\documentclass[runningheads]{llncs}
\usepackage{amsmath,amssymb}
\usepackage{helvet}
\usepackage{courier}
\usepackage{bbm}
\usepackage{color}
\usepackage{paralist}
\usepackage[mathscr]{eucal}
\usepackage{algorithm,algorithmic}

\newcommand{\ind}{\mathbbm{1}}

\newcommand{\param}{\theta}
\newcommand{\TruParam}{{\param_{\star}}}
\newcommand{\Params}{\Theta}
\newcommand{\Priors}{\Pi}

\newcommand{\vc}{d}
\newcommand{\Borel}{{\cal B}}
\newcommand{\BorelX}{\Borel_{\X}}

\newcommand{\Data}{\mathcal{Z}}
\newcommand{\DataX}{\mathbb{X}}
\newcommand{\DataY}{\mathbb{Y}}

\newcommand{\Parity}{{\rm Parity}}

\newcommand{\Ball}{{\rm B}}
\newcommand{\F}{\mathcal{F}}
\newcommand{\C}{\mathbb C}
\newcommand{\D}{\mathcal D}
\newcommand{\X}{\mathcal X}
\renewcommand{\H}{\mathcal H}
\renewcommand{\P}{\mathbb P}
\newcommand{\prior}{\pi}
\newcommand{\refmeas}{\prior_{0}}
\newcommand{\density}{f}
\newcommand{\target}{h^{*}}

\newcommand{\E}{\mathbb E}

\newcommand{\nats}{\mathbb{N}}
\newcommand{\reals}{\mathbb{R}}

\newcommand{\A}{\mathcal A}

\newcommand{\citet}{\cite}
\newcommand{\citep}{\cite}

\newcommand{\argmin}{\mathop{\rm argmin}}

\newcommand{\eps}{\varepsilon}

\newsavebox{\savepar}

\newtheorem{open-problem}{Open Problem}

\begin{document}

\title{Bounds on the Minimax Rate for Estimating a Prior over a VC Class from Independent Learning Tasks}

\titlerunning{Prior Estimation}

\author{Liu Yang \and Steve Hanneke \and Jaime Carbonell}

\authorrunning{Liu Yang, Steve Hanneke, and Jaime Carbonell}

\institute{IBM T.J. Watson Research Center, Yorktown Heights, NY USA.\\
\email{yangli@us.ibm.com}
\and
Princeton, NJ USA.\\ 
\email{steve.hanneke@gmail.com}
\and
Carnegie Mellon University, Pittsburgh, PA USA\\
\email{jgc@cs.cmu.edu}
}

\maketitle

\begin{abstract}
We study the optimal rates of convergence for estimating a prior distribution over a VC class
from a sequence of independent data sets respectively labeled by independent target functions sampled from the prior.
We specifically derive upper and lower bounds on the optimal rates under a smoothness condition on the correct prior,
with the number of samples per data set equal the VC dimension.
These results have implications for the improvements achievable via transfer learning.
We additionally extend this setting to real-valued function, where we establish consistency of 
an estimator for the prior, and discuss an additional application to a preference elicitation problem
in algorithmic economics.
\end{abstract}

\section{Introduction}

In the \emph{transfer learning} setting, we are presented with a sequence of learning problems,
each with some respective target concept we are tasked with learning.  The key question in
transfer learning is how to leverage our access to past learning problems in order to improve
performance on learning problems we will be presented with in the future.  

Among the several proposed models for transfer learning, one particularly appealing model supposes the learning
problems are independent and identically distributed, with unknown distribution, and the advantage
of transfer learning then comes from the ability to estimate this shared distribution based on the
data from past learning problems \citep{baxter:97,yang:13}.  For instance, when customizing a speech recognition 
system to a particular speaker's voice, we might expect the first few people would need to speak many words 
or phrases in order for the system to accurately identify the nuances.  However, after performing this for many
different people, if the software has access to those past training sessions when customizing itself to a new 
user, it should have identified important properties of the speech patterns, such as the common patterns within
each of the major dialects or accents, and other such information about the \emph{distribution} of speech 
patterns within the user population.  It should then be able to leverage this information to reduce the 
number of words or phrases the next user needs to speak in order to train the system, for instance by 
first trying to identify the individual's dialect, then presenting phrases that differentiate common subpatterns 
within that dialect, and so forth.

In analyzing the benefits of transfer learning in such a setting, 
one important question to ask is how quickly we can estimate 
the distribution from which the learning problems are sampled.
In recent work, \citet{yang:13} have shown that under mild conditions
on the family of possible distributions, if the target concepts reside
in a known VC class, then it is possible to estimate this distribtion 
using only a bounded number of training samples per task: 
specifically, a number of samples equal the VC dimension.
However, that work left open the question of quantifying the \emph{rate} of
convergence.  This rate of convergence can have a direct impact
on how much benefit we gain from transfer learning when we are
faced with only a finite sequence of learning problems.  As such, 
it is certainly desirable to derive tight characterizations of this
rate of convergence.

The present work continues that of \citet{yang:13}, bounding the 
rate of convergence for estimating this distribution, under a 
smoothness condition on the distribution.  We derive a generic upper
bound, which holds regardless of the VC class the target concepts
reside in.  The proof of this result builds on that earlier work, but 
requires several interesting innovations to make the rate of convergence
explicit, and to dramatically improve the upper bound implicit in 
the proofs of those earlier results.  We further derive a nontrivial 
lower bound that holds for certain constructed scenarios, which 
illustrates a lower limit on how good of a general upper bound we 
might hope for in results expressed only in terms of 
the number of tasks, the smoothness conditions, and the VC dimension.

We additionally include an extension of the results of \cite{yang:13}
to the setting of real-valued functions, establishing consistency 
(at a uniform rate) for an estimator of a prior over any VC subgraph 
class.  In addition to the application to transfer learning, analogous
to the original work of \cite{yang:13}, we also discuss an application
of this result to a preference elicitation problem in algorithmic economics, 
in which we are tasked with allocating items to a sequence of customers
to approximately maximize the customers' satisfaction, while permitted
access to the customer valuation functions only via value queries.

\section{The Setting}

Let $(\X,\Borel_{\X})$ be a measurable
space \citep{schervish:95} (where $\X$ is called the \emph{instance space}), 
and let $\D$ be a distribution on $\X$ (called the \emph{data distribution}).
Let $\C$ be a VC class of measurable classifiers $h : \X \to \{-1,+1\}$ (called the \emph{concept space}),
and denote by $\vc$ the VC dimension of $\C$ \citep{vapnik:71}.
We suppose $\C$ is equipped with its Borel $\sigma$-algebra $\Borel$ induced by the 
pseudo-metric $\rho(h,g) = \D(\{x \in \X : h(x) \neq g(x)\})$.  Though our results can be formulated 
for general $\D$ (with somewhat more complicated theorem statements), to simplify the statement of 
results we suppose $\rho$ is actually a \emph{metric}, which would follow from appropriate
topological conditions on $\C$ relative to $\D$.

For any two probability measures $\mu_1, \mu_2$ on a measurable space $(\Omega, \F)$, 
define the total variation distance 
\[\|\mu_1 - \mu_2\| = \sup_{A \in \F} \mu_1(A) - \mu_2(A).\]
For a set function $\mu$ on a \emph{finite} measurable space $(\Omega,\F)$, 
we abbreviate $\mu(\omega) = \mu(\{\omega\})$, $\forall \omega \in \Omega$.
Let $\Priors_{\Params} = \{\prior_{\param} : \param \in \Params\}$ 
be a family of probability measures on $\C$ (called \emph{priors}), 
where $\Params$ is an arbitrary index set (called the \emph{parameter space}).
We suppose there exists a probability measure $\prior_{0}$ on $\C$ (the \emph{reference measure})
such that every $\prior_{\param}$ is absolutely continuous with respect to $\prior_{0}$, 
and therefore has a density function $\density_{\param}$ given by the Radon-Nikodym derivative $\frac{{\rm d} \prior_{\param}}{{\rm d} \prior_{0}}$ \citep{schervish:95}.

We consider the following type of estimation problem.
There is a collection of $\C$-valued random variables $\{\target_{t \param} : t \in \nats, \param \in \Params\}$,
where for any fixed $\param \in \Params$ the $\{\target_{t\param}\}_{t=1}^{\infty}$ variables are i.i.d. with distribution $\prior_{\param}$.
For each $\param \in \Params$, there is a sequence $\Data^{t}(\param) = \{(X_{t1},Y_{t1}(\param)), (X_{t2},Y_{t2}(\param)), \ldots\}$,
where $\{X_{ti}\}_{t,i \in \nats}$ are i.i.d. $\D$, and for each $t,i \in \nats$, $Y_{ti}(\param) = \target_{t \param}(X_{ti})$.
We additionally denote by $\Data^{t}_{k}(\param) = \{(X_{t1},Y_{t1}(\param)),\ldots,(X_{tk},Y_{tk}(\param))\}$ the first $k$ elements of $\Data^{t}(\param)$,
for any $k \in \nats$, and similarly $\DataX_{t k} = \{X_{t 1}, \ldots, X_{t k}\}$ and $\DataY_{t k}(\param) = \{Y_{t 1}(\param), \ldots, Y_{t k}(\param)\}$.
Following the terminology used in the transfer learning literature, we refer to the collection of variables associated with each $t$ collectively as the $t^{{\rm th}}$ \emph{task}.
We will be concerned with sequences of estimators $\hat{\param}_{T \param} = \hat{\param}_{T}( \Data^{1}_{k}(\param), \ldots, \Data^{T}_{k}(\param))$, for $T \in \nats$,
which are based on only a bounded number $k$ of samples per task, among the first $T$ tasks.  Our main results specifically study the case of $\vc$ samples per task.
For any such estimator, we measure the \emph{risk} as $\E\left[ \| \prior_{\hat{\param}_{T \TruParam}} - \prior_{\TruParam} \|\right]$,
and will be particularly interested in upper-bounding the worst-case risk $\sup_{\TruParam \in \Params} \E\left[ \| \prior_{\hat{\param}_{T \TruParam}} - \prior_{\TruParam} \| \right]$ 
as a function of $T$, and lower-bounding the minimum possible value of this worst-case risk over all possible $\hat{\param}_{T}$ estimators
(called the \emph{minimax risk}).

In previous work, \citet{yang:13} showed that, if $\Priors_{\Params}$ is a totally bounded family, 
then even with only $\vc$ number of samples per task, the minimax risk (as a function of the number of tasks $T$)
converges to zero.  In fact, that work also proved this is not necessarily the case in general for any number of samples less than $\vc$.
However, the actual rates of convergence were not explicitly derived in that work,
and indeed the upper bounds on the rates of convergence implicit in that analysis 
may often have fairly complicated dependences on $\C$, $\Priors_{\Params}$, and $\D$,
and furthermore often provide only very slow rates of convergence.

To derive explicit bounds on the rates of convergence, 
in the present work we specifically focus on families of \emph{smooth} densities.
The motivation for involving a notion of smoothness in characterizing rates of convergence
is clear if we consider the extreme case in which $\Priors_{\Params}$ contains two priors 
$\prior_{1}$ and $\prior_{2}$, with $\prior_{1}(\{h\}) = \prior_{2}(\{g\}) = 1$, where $\rho(h,g)$
is a very small but nonzero value; in this case, if we have only a small number of samples per task,
we would require many tasks (on the order of $1/\rho(h,g)$) to observe any data points carrying 
any information that would distinguish between these two priors 
(namely, points $x$ with $h(x) \neq g(x)$); yet $\|\prior_1 - \prior_2\| = 1$,
so that we have a slow rate of convergence (at least initially).  A total boundedness condition 
on $\Priors_{\Params}$ would limit the number of such pairs present in $\Priors_{\Params}$, 
so that for instance we cannot have arbitrarily close $h$ and $g$, but less extreme variants
of this can lead to slow asymptotic rates of convergence as well.
Specifically, in the present work we consider the following notion of smoothness.
For $L \in (0,\infty)$ and $\alpha \in (0,1]$, a function $f : \C \to \reals$ is $(L,\alpha)$-H\"{o}lder smooth if
\[\forall h,g \in \C, |f(h) - f(g)| \leq L \rho(h,g)^{\alpha}.\]

\section{An Upper Bound}

We now have the following theorem, holding for an arbitrary VC class $\C$ and data distribution $\D$; it is the main result of this work.

\begin{theorem}
\label{thm:upper-bound}
For $\Priors_{\Params}$ any class of priors on $\C$ having $(L,\alpha)$-H\"{o}lder smooth densities $\{\density_{\param} : \param \in \Params\}$,
for any $T \in \nats$, there exists an estimator $\hat{\param}_{T \param} = \hat{\param}_{T}(\Data^{1}_{\vc}(\param),\ldots, \Data^{T}_{\vc}(\param))$
such that
\begin{equation*}
\sup_{\TruParam \in \Params} \E \| \prior_{\hat{\param}_{T}} - \prior_{\TruParam} \|
= \tilde{O}\left( L T^{- \frac{\alpha^{2}}{2(\vc+2\alpha)(\alpha+2(\vc+1))}} \right).
\end{equation*}
\end{theorem}
\begin{proof} 
By the standard PAC analysis \citep{vapnik:82,blumer:89}, 
for any $\gamma > 0$, with probability greater than $1-\gamma$,
a sample of $k = O((\vc/\gamma)\log(1/\gamma))$ random points
will partition $\C$ into regions of width less than $\gamma$ (under $L_{1}(\D)$).
For brevity, we omit the $t$ subscripts and superscripts on quantities such as $\Data^{t}_{k}(\param)$ throughout the following analysis,
since the claims hold for any arbitrary value of $t$.

For any $\param \in \Params$, let $\prior_{\param}^{\prime}$ denote a (conditional on $X_1,\ldots,X_k$) distribution
defined as follows.  Let $\density_{\param}^{\prime}$ denote the (conditional on $X_1,\ldots,X_k$) density function of $\prior_{\param}^{\prime}$ 
with respect to $\refmeas$, and for any $g \in \C$, let 
$\density_{\param}^{\prime}(g) = \frac{\prior_{\param}(\{h \in \C : \forall i \leq k, h(X_i)=g(X_i)\})}{\refmeas(\{h \in \C : \forall i \leq k, h(X_i)=g(X_i)\})}$
(or $0$ if $\refmeas(\{h \in \C : \forall i \leq k, h(X_i)=g(X_i)\}) = 0$).
In other words, $\prior_{\param}^{\prime}$ has the same probability mass as $\prior_{\param}$ for each of the equivalence classes induced by $X_1,\ldots,X_k$,
but conditioned on the equivalence class, simply has a constant-density distribution over that equivalence class.
Note that every $h \in \C$ has
$\density_{\param}^{\prime}(h)$ between the smallest and largest values of $\density_{\param}(g)$ among $g \in \C$ with $\forall i \leq k, g(X_i) = h(X_i)$;
therefore, by the smoothness condition,
on the event (of probability greater than $1-\gamma$) that each of these regions has diameter less than $\gamma$, 
we have
$\forall h \in \C, |\density_{\param}(h) - \density_{\param}^{\prime}(h)| < L \gamma^{\alpha}$.
On this event, for any $\param,\param^{\prime} \in \Params$,
\begin{equation*}
\| \prior_{\param} - \prior_{\param^{\prime}}\| 
= (1/2)\int |\density_{\param} - \density_{\param^{\prime}}| {\rm d}\refmeas
< L\gamma^{\alpha} + (1/2)\int |\density_{\param}^{\prime} - \density_{\param^{\prime}}^{\prime}| {\rm d}\refmeas.
\end{equation*}
Furthermore, since the regions that define $\density_{\param}^{\prime}$ and $\density_{\param^{\prime}}^{\prime}$
are the same (namely, the partition induced by $X_1,\ldots,X_k$), we have
\begin{align*}
(1/2) \int |\density_{\param}^{\prime} - \density_{\param^{\prime}}^{\prime}| {\rm d}\refmeas
& = (1/2) \!\!\!\!\!\!\!\!\!\sum_{y_1,\ldots,y_k \in \{-1,+1\}}\!\!\! | \prior_{\param}(\{h\in\C : \forall i \leq k, h(X_i) =y_i\}) 
\\ & \phantom{aaaaaaaaaaaaaaaaa} - \prior_{\param^{\prime}}(\{h\in\C : \forall i \leq k, h(X_i) = y_i\})|
\\ & = \| \P_{\DataY_{k}(\param)|\DataX_{k}} - \P_{\DataY_{k}(\param^{\prime})|\DataX_{k}}\|.
\end{align*}
Thus, we have that with probability at least $1-\gamma$,
\[
\| \prior_{\param} - \prior_{\param^{\prime}}\| 
< L\gamma^{\alpha} + \| \P_{\DataY_{k}(\param) | \DataX_{k}} - \P_{\DataY_{k}(\param^{\prime})| \DataX_{k}} \|.
\]

Following analogous to the inductive argument of \citet{yang:13}, 
suppose $I \subseteq \{1,\ldots,k\}$, fix $\bar{x}_{I} \in \X^{|I|}$ and $\bar{y}_{I} \in \{-1,+1\}^{|I|}$.
Then the $\tilde{y}_{I} \in \{-1,+1\}^{|I|}$ for which $\|\bar{y}_{I}-\tilde{y}_{I}\|_{1}$ is minimal, 
subject to the constraint that no $h \in \C$ has $h(\bar{x}_{I}) = \tilde{y}_{I}$,
has $(1/2)\|\bar{y}_{I}-\tilde{y}_{I}\|_{1} \leq \vc+1$;
also, for any $i \in I$ with $\bar{y}_{i} \neq \tilde{y}_{i}$, letting $\bar{y}^{\prime}_{j} = \bar{y}_{j}$ for $j \in I\setminus\{i\}$
and $\bar{y}^{\prime}_{i} = \tilde{y}_{i}$, we have
\begin{equation*}
\P_{\DataY_{I}(\param)|\DataX_{I}}(\bar{y}_{I} | \bar{x}_{I})
= \P_{\DataY_{I\setminus\{i\}}(\param)|\DataX_{I\setminus\{i\}}}(\bar{y}_{I\setminus\{i\}} | \bar{x}_{I\setminus\{i\}})
- \P_{\DataY_{I}(\param)|\DataX_{I}}(\bar{y}^{\prime}_{I} | \bar{x}_{I}),
\end{equation*}
and similarly for $\param^{\prime}$, so that
\begin{align*}
& | \P_{\DataY_{I}(\param)|\DataX_{I}}(\bar{y}_{I} | \bar{x}_{I})- \P_{\DataY_{I}(\param^{\prime})|\DataX_{I}}(\bar{y}_{I}|\bar{x}_{I}) |
\\ & \leq | \P_{\DataY_{I\setminus\{i\}}(\param)|\DataX_{I\setminus\{i\}}}(\bar{y}_{I\setminus\{i\}} | \bar{x}_{I\setminus\{i\}})
- \P_{\DataY_{I\setminus\{i\}}(\param^{\prime})|\DataX_{I\setminus\{i\}}}(\bar{y}_{I\setminus\{i\}}|\bar{x}_{I\setminus\{i\}}) |
\\ & \phantom{aaaa}+ | \P_{\DataY_{I}(\param)|\DataX_{I}}(\bar{y}^{\prime}_{I} | \bar{x}_{I}) - \P_{\DataY_{I}(\param^{\prime})|\DataX_{I}}(\bar{y}^{\prime}_{I} | \bar{x}_{I}) |.
\end{align*}
Now consider that these two terms inductively define a binary tree.
Every time the tree branches left once, it arrives at a difference of probabilities for a set $I$ of one less element than that of its parent.
Every time the tree branches right once, it arrives at a difference of probabilities for a $\bar{y}_{I}$ one closer to an unrealized $\tilde{y}_{I}$
than that of its parent.
Say we stop branching the tree upon reaching a set $I$ and a $\bar{y}_{I}$ such that
either $\bar{y}_{I}$ is an unrealized labeling, or $|I|=\vc$.  Thus, we can bound the original (root node) difference of 
probabilities by the sum of the differences of probabilities for the leaf nodes with $|I|=\vc$.
Any path in the tree can branch left at most $k-\vc$ times (total) before reaching a set $I$ with only $\vc$ elements, 
and can branch right at most $\vc+1$ times in a row before reaching a $\bar{y}_{I}$ such that both probabilities are zero,
so that the difference is zero.  So the depth of any leaf node with $|I|=\vc$ is at most $(k-\vc)\vc$.
Furthermore, at any level of the tree, from left to right the nodes have strictly decreasing $|I|$ values,
so that the maximum width of the tree is at most $k-\vc$.  So the total number of leaf nodes with $|I|=\vc$
is at most $(k-\vc)^{2}\vc$.
Thus, for any $\bar{y} \in \{-1,+1\}^{k}$ and $\bar{x} \in \X^{k}$,
\begin{align*}
& | \P_{\DataY_{k}(\param)|\DataX_{k}}(\bar{y} | \bar{x})- \P_{\DataY_{k}(\param^{\prime})|\DataX_{k}}(\bar{y}|\bar{x})| 
\\ & \leq (k-\vc)^{2}\vc \cdot \max_{\bar{y}^{\vc} \in \{-1,+1\}^{\vc}} \max_{D \in \{1,\ldots,k\}^{\vc}} | \P_{\DataY_{\vc}(\param)|\DataX_{\vc}}(\bar{y}^{\vc} | \bar{x}_{D})-
\P_{\DataY_{\vc}(\param^{\prime})|\DataX_{\vc}}(\bar{y}^{\vc}|\bar{x}_{D})|.
\end{align*}
Since
\begin{equation*}
\|\P_{\DataY_{k}(\param) | \DataX_{k}}-\P_{\DataY_{k}(\param^{\prime})|\DataX_{k}}\|
= (1/2) \sum_{\bar{y}^{k} \in \{-1,+1\}^{k}} | \P_{\DataY_{k}(\param)|\DataX_{k}}(\bar{y}^{k}) - \P_{\DataY_{k}(\param^{\prime})|\DataX_{k}}(\bar{y}^{k})|,
\end{equation*}
and by Sauer's Lemma this is at most
\[ 
(ek)^{\vc} \max_{\bar{y}^{k} \in \{-1,+1\}^{k}} | \P_{\DataY_{k}(\param)|\DataX_{k}}(\bar{y}^{k}) - \P_{\DataY_{k}(\param^{\prime})|\DataX_{k}}(\bar{y}^{k})|,
\]
we have that 
\begin{align*}
&\|\P_{\DataY_{k}(\param) | \DataX_{k}}-\P_{\DataY_{k}(\param^{\prime})|\DataX_{k}}\|
\\ & \leq (ek)^{\vc} k^{2}\vc \max_{\bar{y}^{\vc} \in \{-1,+1\}^{\vc}} \max_{D \in \{1,\ldots,k\}^{\vc}} | \P_{\DataY_{\vc}(\param)|\DataX_{D}}(\bar{y}^{\vc}) - 
\P_{\DataY_{\vc}(\param^{\prime})|\DataX_{D}}(\bar{y}^{\vc})|.
\end{align*}
Thus, we have that
\begin{align*}
& \| \prior_{\param} - \prior_{\param^{\prime}} \| = \E \| \prior_{\param} - \prior_{\param^{\prime}} \|
\\ & < \gamma \!+\! L \gamma^{\alpha} \!+\! 
(ek)^{\vc} k^{2}\vc \E\bigg[ \max_{\bar{y}^{\vc} \in \{-1,+1\}^{\vc}} \max_{D \in \{1,\ldots,k\}^{\vc}} 
\P_{\DataY_{\vc}(\param)|\DataX_{D}}(\bar{y}^{\vc}) - \P_{\DataY_{\vc}(\param^{\prime})|\DataX_{D}}(\bar{y}^{\vc})|\bigg].
\end{align*}
Note that
\begin{align*}
& \E\bigg[ \max_{\bar{y}^{\vc} \in \{-1,+1\}^{\vc}} \max_{D \in \{1,\ldots,k\}^{\vc}} 
| \P_{\DataY_{\vc}(\param)|\DataX_{D}}(\bar{y}^{\vc}) - \P_{\DataY_{\vc}(\param^{\prime})|\DataX_{D}}(\bar{y}^{\vc})|\bigg]
\\ & \leq \sum_{\bar{y}^{\vc} \in \{-1,+1\}^{\vc}} \sum_{D \in \{1,\ldots,k\}^{\vc}} \E\big[ | \P_{\DataY_{\vc}(\param)|\DataX_{D}}(\bar{y}^{\vc}) 
- \P_{\DataY_{\vc}(\param^{\prime})|\DataX_{D}}(\bar{y}^{\vc})|\big]
\\ & \leq (2k)^{\vc} \max_{\bar{y}^{\vc} \in \{-1,+1\}^{\vc}} \max_{D \in \{1,\ldots,k\}^{\vc}} \E\big[ | \P_{\DataY_{\vc}(\param)|\DataX_{D}}(\bar{y}^{\vc}) 
- \P_{\DataY_{\vc}(\param^{\prime})|\DataX_{D}}(\bar{y}^{\vc})|\big],
\end{align*}
and by exchangeability, this last line equals
\[
(2k)^{\vc} \max_{\bar{y}^{\vc} \in \{-1,+1\}^{\vc}} \E\left[ | \P_{\DataY_{\vc}(\param)|\DataX_{\vc}}(\bar{y}^{\vc}) - \P_{\DataY_{\vc}(\param^{\prime})|\DataX_{\vc}}(\bar{y}^{\vc})|\right].
\]
\citet{yang:13} showed that 
$\E\left[ | \P_{\DataY_{\vc}(\param)|\DataX_{\vc}}(\bar{y}^{\vc}) - \P_{\DataY_{\vc}(\param^{\prime})|\DataX_{\vc}}(\bar{y}^{\vc})|\right] 
\leq 4 \sqrt{ \| \P_{\Data_{\vc}(\param)} - \P_{\Data_{\vc}(\param^{\prime})} \|}$,
so that in total
we have
$\| \prior_{\param} - \prior_{\param^{\prime}}\| 
\!<\! (L\!+\!1)\gamma^{\alpha} \!+\! 4 (2ek)^{2\vc+2}\!\!\sqrt{ \| \P_{\Data_{\vc}(\param)} \!-\! \P_{\Data_{\vc}(\param^{\prime})}\|}$.
Plugging in the value of $k = c(\vc/\gamma)\log(1/\gamma)$, this is
\[
(L\!+\!1)\gamma^{\alpha} + 4\! \left(\!2ec\frac{\vc}{\gamma}\log\!\left(\frac{1}{\gamma}\right)\!\right)^{\!\!2\vc+2} \!\sqrt{\| \P_{\Data_{\vc}(\param)} \!-\! \P_{\Data_{\vc}(\param^{\prime})}\|}.
\]

Thus, it suffices to bound the rate of convergence (in total variation distance) of some estimator of $\P_{\Data_{\vc}(\TruParam)}$.
If $N(\eps)$ is the $\eps$-covering number of $\{\P_{\Data_{\vc}(\param)} : \param \in \Params\}$, 
then taking $\hat{\param}_{T \TruParam}$ as the minimum distance skeleton estimate of \citet{yatracos:85,devroye:01} 
achieves expected total variation distance $\eps$ from $\P_{\Data_{\vc}(\TruParam)}$,
for some $T = O((1/\eps^{2}) \log N(\eps/4))$.
We can partition $\C$ into $O( (L/\eps)^{\vc/\alpha} )$ cells of diameter $O((\eps/L)^{1/\alpha})$,
and set a constant density value within each cell,
on an $O(\eps)$-grid of density values, 
and every prior with $(L,\alpha)$-H\"{o}lder smooth density will have density 
within $\eps$ of some density so-constructed; there are then at most $(1/\eps)^{O((L/\eps)^{\vc/\alpha})}$ such densities,
so this bounds the covering numbers of $\Priors_{\Params}$.
Furthermore, the covering number of $\Priors_{\Params}$ upper bounds 
$N(\eps)$ \citep{yang:13}, so that $N(\eps) \leq (1/\eps)^{O((L/\eps)^{\vc/\alpha})}$.

Solving $T \!=\! O(\eps^{-2} (L/\eps)^{\vc/\alpha} \log (1 / \eps) )$ for $\eps$, 
we have $\eps \!=\! O\!\left( L \!\left(\frac{\log(TL)}{T}\right)^{\frac{\alpha}{\vc+2\alpha}} \right)$.
So this bounds the rate of convergence for $\E\| \P_{\Data_{\vc}(\hat{\param}_{T})} - \P_{\Data_{\vc}(\TruParam)}\|$,
for $\hat{\param}_{T}$ the minimum distance skeleton estimate.
Plugging this rate into the bound on the priors, combined with Jensen's inequality, we have 
\begin{equation*}
\E \| \prior_{\hat{\param}_{T}} - \prior_{\TruParam} \| 
< (L+1) \gamma^{\alpha}
+ 4 \left(2ec\frac{\vc}{\gamma}\log\left(\frac{1}{\gamma}\right)\right)^{2\vc+2} \!\!\!\!\!\times O\left( L \left(\frac{\log(TL)}{T}\right)^{\frac{\alpha}{2\vc+4\alpha}} \right).
\end{equation*}
This holds for any $\gamma > 0$, so minimizing this expression over $\gamma > 0$ yields a bound on the rate.
For instance, with $\gamma = \tilde{O}\left(T^{- \frac{\alpha}{2(\vc+2\alpha)(\alpha + 2(\vc+1))}}\right)$,
we have
\[ 
\E \| \prior_{\hat{\param}_{T}} - \prior_{\TruParam} \|
= \tilde{O}\left( L T^{-\frac{\alpha^{2}}{2(\vc+2\alpha)(\alpha+2(\vc+1))}} \right).
\]
\qed
\end{proof}

\section{A Minimax Lower Bound}

One natural quesiton is whether Theorem~\ref{thm:upper-bound} can generally be improved.
While we expect this to be true for some fixed VC classes (e.g., those of finite size),
and in any case we expect that some of the constant factors in the exponent may be improvable,
it is not at this time clear whether the general form of $T^{-\Theta(\alpha^2 / (d+\alpha)^2)}$ is sometimes optimal.
One way to investigate this question is to construct specific spaces $\C$ and distributions $\D$ 
for which a lower bound can be obtained.  
In particular, we are generally interested in exhibiting lower bounds that are worse than 
those that apply to the usual problem of density estimation based on direct access to 
the $\target_{t \TruParam}$ values (see Theorem~\ref{thm:absolute-lower-bound} below).

Here we present a lower bound that is interesting for this reason.
However, although larger than the optimal rate for methods with direct 
access to the target concepts, it is still far from matching the upper bound
above, so that the question of tightness remains open.  Specifically, we have the following result.

\begin{theorem}
\label{thm:d-sample-lower-bound}
For any integer $\vc \geq 1$, any $L>0, \alpha\in (0,1]$, there is a value $C(\vc,L,\alpha) \in (0,\infty)$ such that,
for any $T \in \nats$, there exists 
an instance space $\X$, 
a concept space $\C$ of VC dimension $\vc$,
a distribution $\D$ over $\X$,
and a distribution $\prior_{0}$ over $\C$ such that,
for $\Priors_{\Params}$ a set of distributions over $\C$ with $(L,\alpha)$-H\"{o}lder smooth density functions with respect to $\prior_{0}$,
any estimator
$\hat{\param}_{T} = \hat{\param}_{T}(\Data^{1}_{\vc}(\TruParam),\ldots,\Data^{T}_{\vc}(\TruParam))$
has 
\begin{equation*}
\sup_{\TruParam \in \Params} \E\left[ \| \prior_{\hat{\param}_{T}} - \prior_{\TruParam} \| \right] \geq C(\vc,L,\alpha) T^{- \frac{\alpha}{2(\vc + \alpha)}}.
\end{equation*}
\end{theorem}
\begin{proof}(Sketch)
We proceed by a reduction from the task of determining the bias of a coin from among two given possibilities.
Specifically, fix any $\gamma \in (0,1/2)$, $n \in \nats$, and let $B_{1}(p),\ldots,B_{n}(p)$ be i.i.d ${\rm Bernoulli}(p)$ random variables, for each $p \in [0,1]$;
then it is known that, for any (possibly nondeterministic) decision rule $\hat{p}_{n}: \{0,1\}^{n} \to \{(1+\gamma)/2, (1-\gamma)/2\}$, 
\begin{multline}
\label{eqn:coin-bias-lower}
\frac{1}{2} \sum_{p \in \{(1+\gamma)/2,(1-\gamma)/2\}} \P( \hat{p}_{n}(B_1(p),\ldots,B_{n}(p)) \neq p ) 
\\ \geq (1/32) \cdot \exp\left\{-128 \gamma^{2} n / 3\right\}.
\end{multline}
This easily follows from the results of \citet{bar-yossef:03}, combined with a result of \citet{poland:06} bounding the KL divergence
(see also \cite{wald:45})

To use this result, we construct a learning problem as follows.
Fix some $m \in \nats$ with $m \geq \vc$, let 
$\X = \{1,\ldots,m\}$,
and let $\C$ be the space of all classifiers $h : \X \to \{-1,+1\}$ such that 
$|\{x \in \X : h(x) = +1\}| \leq \vc$.
Clearly the VC dimension of $\C$ is $\vc$.
Define the distribution $\D$ 
as uniform over $\X$.
Finally, we specify a family of $(L,\alpha)$-H\"{o}lder smooth priors, parameterized by $\Params = \{-1,+1\}^{\binom{m}{\vc}}$, as follows.
Let $\gamma_{m} = (L/2)(1/m)^{\alpha}$.
First, enumerate the $\binom{m}{\vc}$ distinct $d$-sized subsets of $\{1,\ldots,m\}$ as $\X_1,\X_2,\ldots,\X_{\binom{m}{\vc}}$.
Define the reference distribution $\prior_{0}$ by the property that, for any $h \in \C$, letting $q = |\{x : h(x)=+1\}|$, $\prior_{0}(\{h\}) = (\frac{1}{2})^{\vc} \binom{m-q}{\vc-q} / \binom{m}{\vc}$.
For any $\mathbf{b} = (b_1,\ldots,b_{\binom{m}{\vc}}) \in \{-1,1\}^{\binom{m}{\vc}}$, define the prior $\prior_{\mathbf{b}}$ as the distribution of a random variable $h_{\mathbf{b}}$ specified by the following generative model.
Let $i^{*} \sim {\rm Uniform}(\{1,\ldots,\binom{m}{\vc}\})$, let $C_{\mathbf{b}}(i^*) \sim {\rm Bernoulli}((1+\gamma_{m} b_{i^*})/2)$;
finally, $h_{\mathbf{b}} \sim {\rm Uniform}(\{ h \in \C : \{x : h(x)=+1\} \subseteq \X_{i^*}, \Parity(|\{x : h(x)=+1\}|)=C_{\mathbf{b}}(i^*)\})$, where $\Parity(n)$ is $1$ if $n$ is odd, or $0$ if $n$ is even.
We will refer to the variables in this generative model below.
For any $h \in \C$, letting $H = \{x : h(x)=+1\}$ and $q = |H|$, we can equivalently express 
$\prior_{\mathbf{b}}(\{h\}) = (\frac{1}{2})^{\vc} \binom{m}{\vc}^{-1} \sum_{i=1}^{\binom{m}{\vc}} \ind[ H \subseteq \X_{i} ] (1+\gamma_m b_i)^{\Parity(q)} (1-\gamma_m b_i)^{1-\Parity(q)}$.
From this explicit representation, it is clear that, letting $\density_{\mathbf{b}} = \frac{{\rm d} \prior_{\mathbf{b}}}{{\rm d} \prior_{0}}$, 
we have $\density_{\mathbf{b}}(h) \in [1-\gamma_m,1+\gamma_m]$ for all $h \in \C$.
The fact that $\density_{\mathbf{b}}$ is H\"{o}lder smooth follows from this, 
since every distinct $h,g \in \C$ have 
$\D(\{x : h(x) \neq g(x)\}) \geq 1/m = (2\gamma_{m} / L)^{1/\alpha}$.

Next we set up the reduction as follows.
For any estimator $\hat{\prior}_{T} = \hat{\prior}_{T}(\Data^{1}_{\vc}(\TruParam),$ $\ldots,\Data^{T}_{\vc}(\TruParam))$, 
and each $i \in \{1,\ldots,\binom{m}{\vc}\}$, let $h_{i}$ be the classifier with $\{x : h_{i}(x)=+1\} = \X_i$;
also, if $\hat{\prior}_{T}(\{h_i\}) > (\frac{1}{2})^{\vc} / \binom{m}{d}$, let $\hat{b}_{i} = 2\Parity(\vc)-1$,
and otherwise $\hat{b}_{i} = 1 - 2\Parity(\vc)$.
We use these $\hat{b}_i$ values to estimate the original $b_i$ values.
Specifically, let $\hat{p}_i = (1 + \gamma_{m} \hat{b}_{i})/2$ and $p_i = (1+\gamma_{m} b_i)/2$, where $\mathbf{b} = \TruParam$.
Then 
\begin{align*}
\| \hat{\prior}_{T} - \prior_{\TruParam} \| 
& \geq (1/2)\sum_{i=1}^{\binom{m}{\vc}} | \hat{\prior}_{T}(\{h_i\}) - \prior_{\TruParam}(\{h_i\}) | 
\\ & \geq (1/2)\sum_{i=1}^{\binom{m}{\vc}} \frac{\gamma_{m}}{2^{\vc} \binom{m}{\vc}} | \hat{b}_{i} - b_{i} | / 2
= (1/2) \sum_{i=1}^{\binom{m}{\vc}} \frac{1}{2^{\vc} \binom{m}{\vc}} | \hat{p}_{i} - p_{i} |.
\end{align*}
Thus, we have reduced from the problem of deciding the biases of these $\binom{m}{\vc}$ independent Bernoulli random variables.
To complete the proof, it suffices to lower bound the expectation of the right side for an \emph{arbitrary} estimator.

Toward this end, we in fact study an even easier problem.  Specifically, consider an estimator $\hat{q}_{i} = \hat{q}_{i}(\Data^{1}_{\vc}(\TruParam),\ldots,\Data^{T}_{\vc}(\TruParam), i_1^*, \ldots, i_T^*)$,
where $i_t^*$ is the $i^*$ random variable in the generative model that defines $\target_{t \TruParam}$;
that is, $i_t^{*} \sim {\rm Uniform}(\{1,$ $\ldots,\binom{m}{\vc}\})$, $C_{t} \sim {\rm Bernoulli}((1+\gamma_{m} b_{i_t^*})/2)$, and $\target_{t \TruParam} \sim {\rm Uniform}(\{ h \in \C : \{x : h(x)=+1\} \subseteq \X_{i_t^*}, \Parity(|\{x : h(x)=+1\}|)=C_{t}\})$,
where the $i_t^*$ are independent across $t$, as are the $C_{t}$ and $\target_{t \TruParam}$.
Clearly the $\hat{p}_{i}$ from above can be viewed as an estimator of this type, which simply ignores the knowledge of $i_t^*$.
The knowledge of these $i_t^*$ variables simplifies the analysis, since given $\{i_t^* : t \leq T\}$,
the data can be partitioned into $\binom{m}{\vc}$ disjoint sets, $\{\{\Data^{t}_{\vc}(\TruParam) : i_t^* = i\} : i = 1,\ldots,\binom{m}{\vc}\}$,
and we can use only the set $\{\Data^{t}_{\vc}(\TruParam) : i_t^* = i\}$ to estimate $p_i$.
Furthermore, we can use only the subset of these for which $\DataX_{t \vc} = \X_i$,
since otherwise we have zero information about the value of $\Parity(|\{x:\target_{t \TruParam}(x)=+1\}|)$.
That is, given $i_t^* = i$, any $\Data^{t}_{\vc}(\TruParam)$ is conditionally independent from every $b_j$ for $j \neq i$,
and is even conditionally independent from $b_i$ when $\DataX_{t \vc}$ is not completely contained in $\X_i$;
specifically, in this case, regardless of $b_i$, 
the conditional distribution of $\DataY_{t \vc}(\TruParam)$ given $i_t^* = i$ and given $\DataX_{t \vc}$ is
a product distribution, which deterministically assigns label $-1$ to those $Y_{t k}(\TruParam)$ with $X_{t k} \notin \X_i$,
and gives uniform random values to the subset of $\DataY_{t \vc}(\TruParam)$ with their respective $X_{t k} \in \X_i$.
Finally, letting $r_t = \Parity(|\{k \leq \vc : Y_{t k}(\TruParam) = +1\}|)$,
we note that given $i_t^* = i$, $\DataX_{t \vc} = \X_i$, and the value $r_t$, 
$b_i$ is conditionally independent from $\Data^{t}_{\vc}(\TruParam)$.
Thus, the set of values $C_{i T}(\TruParam) = \{ r_t : i_t^* = i, \DataX_{t \vc} = \X_i\}$ is a sufficient statistic for $b_i$ (hence for $p_i$).
Recall that, when $i_t^* = i$ and $\DataX_{t \vc} = \X_i$, the value of $r_t$ is equal to $C_{t}$, a ${\rm Bernoulli}(p_i)$ random variable.
Thus, we neither lose nor gain anything (in terms of risk) by restricting ourselves to estimators $\hat{q}_{i}$ of the type 
$\hat{q}_{i} = \hat{q}_{i}(\Data^{1}_{\vc}(\TruParam), \ldots, \Data^{T}_{\vc}(\TruParam), i_1^*,\ldots, i_T^*) = \hat{q}_{i}^{\prime}(C_{i T}(\TruParam))$,
for some $\hat{q}_{i}^{\prime}$ \citep{schervish:95}: that is, estimators that are a function of the $N_{i T}(\TruParam) = |C_{i T}(\TruParam)|$ ${\rm Bernoulli}(p_i)$
random variables, which we should note are conditionally i.i.d. given $N_{i T}(\TruParam)$.

Thus, by \eqref{eqn:coin-bias-lower}, for any $n \leq T$, 
\begin{align*}
\frac{1}{2} \sum_{b_i \in \{-1,+1\}} \!\E\left[ | \hat{q}_i - p_i | \Big| N_{i T}(\TruParam) = n \right] 
& = \frac{1}{2} \sum_{b_i \in \{-1,+1\}} \!\gamma_{m} \P\left( \hat{q}_i \neq p_i \Big| N_{i T}(\TruParam) = n\right)
\\ & \geq (\gamma_{m}/32) \cdot \exp\left\{ - 128 \gamma_{m}^{2} N_i / 3 \right\}.
\end{align*}
Also note that, for each $i$, 
$\E[N_i] = \frac{\vc! (1/m)^{\vc}}{\binom{m}{\vc}}T \leq (\vc / m)^{2\vc} T = \vc^{2\vc} (2 \gamma_{m} / L)^{2\vc/\alpha} T$.
Thus, Jensen's inequality, linearity of expectation, and the law of total expectation imply
\[
\frac{1}{2} \sum_{b_i \in \{-1,+1\}} \E\left[ | \hat{q}_i - p_i | \right] 
\geq (\gamma_{m} /32) \cdot \exp\left\{ - 43 (2/L)^{2\vc/\alpha} \vc^{2\vc} \gamma_{m}^{2+ 2\vc/\alpha} T \right\}.
\]
Thus, by linearity of the expectation, 
\begin{align*}
& \left(\frac{1}{2}\right)^{\binom{m}{\vc}} \!\!\!\!\sum_{\mathbf{b} \in \{-1,+1\}^{\binom{m}{\vc}}} \!\!\!\!\E\left[\sum_{i=1}^{\binom{m}{\vc}} \frac{1}{2^{\vc} \binom{m}{\vc}} | \hat{q}_{i} - p_{i} | \right] 
= \sum_{i=1}^{\binom{m}{\vc}} \frac{1}{2^{\vc} \binom{m}{\vc}}  \frac{1}{2} \sum_{b_i \in \{-1,+1\}} \!\!\!\!\E\left[ | \hat{q}_{i} - p_{i} | \right] 
\\ & \geq (\gamma_{m} / (32 \cdot 2^{\vc})) \cdot \exp\left\{ - 43 (2/L)^{2\vc/\alpha} \vc^{2\vc} \gamma_{m}^{2+ 2\vc/\alpha} T \right\}.
\end{align*}
In particular, taking 
$m = \left\lceil (L/2)^{1/\alpha} \left(43 (2/L)^{2\vc/\alpha} \vc^{2\vc} T\right)^{\frac{1}{2(\vc+\alpha)}} \right\rceil$,
we have 
$\gamma_{m} = \Theta\left( \left( 43 (2/L)^{2\vc/\alpha} \vc^{2\vc} T \right)^{- \frac{\alpha}{2(\vc+\alpha)}}\right)$,
so that 
\begin{multline*}
 \left(\frac{1}{2}\right)^{\binom{m}{\vc}} \sum_{\mathbf{b} \in \{-1,+1\}^{\binom{m}{\vc}}} \E\left[\sum_{i=1}^{\binom{m}{\vc}} \frac{1}{2^{\vc} \binom{m}{\vc}} | \hat{q}_{i} - p_{i} | \right] 
\\ = \Omega\left( 2^{-\vc} \left( 43 (2/L)^{2\vc/\alpha} \vc^{2\vc} T \right)^{- \frac{\alpha}{2(\vc+\alpha)}}\right).
\end{multline*}
In particular, this implies there exists some $\mathbf{b}$ for which 
\[
\E\left[\sum_{i=1}^{\binom{m}{\vc}} \frac{1}{2^{\vc} \binom{m}{\vc}} | \hat{q}_{i} - p_{i} | \right] 
= \Omega\left( 2^{-\vc} \left( 43 (2/L)^{2\vc/\alpha} \vc^{2\vc} T \right)^{- \frac{\alpha}{2(\vc+\alpha)}}\right).
\]
Applying this lower bound to the estimator $\hat{p}_i$
above yields the result. \qed
\end{proof}

It is natural to wonder how these rates might potentially improve if we allow $\hat{\param}_{T}$
to depend on more than $\vc$ samples per data set.  To establish limits on such improvements,
we note that in the extreme case of allowing the estimator to depend on the full $\Data^{t}(\TruParam)$
data sets, we may recover the known results lower bounding the risk of density 
estimation from i.i.d. samples from a smooth density, as indicated by the following result.

\begin{theorem}
\label{thm:absolute-lower-bound}
For any integer $\vc \geq 1$, there exists 
an instance space $\X$, 
a concept space $\C$ of VC dimension $\vc$,
a distribution $\D$ over $\X$,
and a distribution $\prior_{0}$ over $\C$ such that,
for $\Priors_{\Params}$ the set of distributions over $\C$ with $(L,\alpha)$-H\"{o}lder smooth density functions with respect to $\prior_{0}$,
any sequence of estimators,
$\hat{\param}_{T} = \hat{\param}_{T}(\Data^{1}(\TruParam),\ldots,\Data^{T}(\TruParam))$ ($T = 1,2,\ldots$),
has 
\begin{equation*}
\sup_{\TruParam \in \Params} \E\left[ \| \prior_{\hat{\param}_{T}} - \prior_{\TruParam} \| \right] = \Omega\left( T^{- \frac{\alpha}{\vc + 2\alpha}} \right).
\end{equation*}
\end{theorem}

The proof is a simple reduction from the problem of estimating $\prior_{\TruParam}$ 
based on direct access to $\target_{1 \TruParam}, \ldots, \target_{T \TruParam}$, 
which is essentially equivalent to the standard model of density estimation, and indeed
the lower bound in Theorem~\ref{thm:absolute-lower-bound} is a well-known result
for density estimation from $T$ i.i.d. samples from a H\"{o}lder smooth density in a 
$\vc$-dimensional space \citep{devroye:01}.

\section{Real-Valued Functions and an Application in Algorithmic Economics}
\label{sec:real-valued}

In this section, we present results generalizing the analysis of \cite{yang:13} to classes of real-valued
functions.  We also present an application of this generalization to a preference elicitation problem.

\subsection{Consistent Estimation of Priors over Real-Valued Functions at a Bounded Rate}

In this section, we let $\Borel$ denote a $\sigma$-algebra on $\X \times \reals$,
and again let $\BorelX$ denote the corresponding $\sigma$-algebra on $\X$.
Also, for measurable functions $h,g : \X \to \reals$, let $\rho(h,g) = \int |h-g| {\rm d}P_{X}$,
where $P_{X}$ is a distribution over $\X$.
Let $\F$ be a class of functions $\X \to \reals$ with Borel $\sigma$-algebra $\Borel_{\F}$ induced by $\rho$.
Let $\Params$ be a set, and for each $\param \in \Params$, let $\prior_{\param}$ denote a probability measure on $(\F,\Borel_{\F})$.
We suppose $\{\prior_{\param} : \param \in \Params\}$ is totally bounded in total variation distance, and that $\F$ is a uniformly bounded VC subgraph class with pseudodimension $\vc$.
We also suppose $\rho$ is a \emph{metric} when restricted to $\F$.

As above, 
let $\{X_{ti}\}_{t,i \in \nats}$ be i.i.d. $P_{X}$ random variables.
For each $\param \in \Params$, let $\{\target_{t \param} \}_{t\in\nats}$ be i.i.d. $\prior_{\param}$ random variables, independent from $\{X_{ti}\}_{t,i\in\nats}$.
For each $t \in \nats$ and $\param \in \Params$,
let $Y_{ti}(\param) = \target_{t \param}(X_{ti})$ for $i \in \nats$,
and let $\Data^{t}(\param) = \{(X_{t1},Y_{t1}(\param)), (X_{t2},Y_{t2}(\param)),\ldots\}$;
for each $k \in \nats$, define $\Data^{t}_{k}(\param) = \{(X_{t1},Y_{t1}(\param)),$ $\ldots, (X_{tk},Y_{tk}(\param))\}$,
$\DataX_{tk} = \{X_{t1},\ldots,X_{tk}\}$, and $\DataY_{tk}(\param) = \{Y_{t1}(\param),\ldots,Y_{tk}(\param)\}$.

We have the following result.
The proof parallels that of \cite{yang:13} (who studied the special case of binary functions),
with a few important twists (in particular, a significantly different approach in the analogue of their Lemma 3).
The details are included in Appendix~\ref{app:real-valued}.

\begin{theorem}
\label{thm:consistency}
There exists an estimator $\hat{\param}_{T \TruParam} = \hat{\param}_{T}(\Data^{1}_{\vc}(\TruParam),\ldots,\Data^{T}_{\vc}(\TruParam))$,
and functions $R : \nats_{0} \times (0,1] \to [0,\infty)$ and $\delta : \nats_{0} \times (0,1] \to [0,1]$ such that, 
for any $\alpha > 0$, $\lim\limits_{T\to\infty} R(T,\alpha) = \lim\limits_{T\to\infty} \delta(T,\alpha) = 0$ and for any $T \in \nats_{0}$
and $\TruParam \in \Params$, 
\[
\P\left( \| \prior_{\hat{\param}_{T \TruParam}} - \prior_{\TruParam} \| > R(T,\alpha) \right) \leq \delta(T,\alpha) \leq \alpha.
\]
\end{theorem}

\subsection{Maximizing Customer Satisfaction in Combinatorial Auctions}
\label{sec:satisfaction}

Theorem~\ref{thm:consistency} has a clear application in the context of transfer learning,
following analogous arguments to those given in the special case of binary classification by \cite{yang:13}.
In addition to that application, we can also use Theorem~\ref{thm:consistency} in the context of 
the following problem in algorithmic economics, where the objective is to serve a sequence of 
customers so as to maximize their satisfaction.

Consider an online travel agency, where customers go to the site with some idea of what type 
of travel they are interested in; the site then poses a series of questions to each customer,
and identifies a travel package that best suits their desires, budget, and dates.  There are 
many options of travel packages, with options on location, site-seeing tours, hotel and room quality,
etc.  Because of this, serving the needs of an \emph{arbitrary} customer might be a lengthy process,
requiring many detailed questions.  Fortunately, the stream of customers is typically not a worst-case
sequence, and in particular obeys many statistical regularities: in particular, it is not too
far from reality to think of the customers as being independent and identically distributed samples.
With this assumption in mind, it becomes desirable to identify some of these statistical regularities
so that we can pose the questions that are typically most relevant, and thereby more quickly identify
the travel package that best suits the needs of the typical customer.  One straightforward
way to do this is to directly \emph{estimate} the distribution of customer value functions,
and optimize the questioning system to minimize the expected number of questions needed to 
find a suitable travel package.

One can model this problem in the style of Bayesian combinatorial auctions, in which 
each customer has a value function for each possible bundle of items.  However, it is 
slightly different, in that we do not assume the distribution of customers is known,
but rather are interested in estimating this distribution; the obtained estimate can then 
be used in combination with methods based on Bayesian decision theory.  In contrast to 
the literature on Bayesian auctions (and subjectivist Bayesian decision theory in general), 
this technique is able to maintain general guarantees on performance that hold under an 
objective interpretation of the problem, rather than merely guarantees holding under an 
arbitrary assumed prior belief.
This general idea is sometimes referred to as \emph{Empirical Bayesian} decision theory
in the machine learning and statistics literatures.
The ideal result for an Empirical Bayesian algorithm is to 
be competitive with the corresponding Bayesian methods based on the \emph{actual} distribution
of the data (assuming the data are random, with an unknown distribution);
that is, although the Empirical Bayesian methods only operate with a data-based estimate of the 
distribution, the aim is to perform nearly as well as methods based on the true (unobservable)
distribution.  In this work, we present results of this type, in the context of an abstraction of 
the aforementioned online travel agency problem, where the measure of performance is the 
expected number of questions to find a suitable package.

The specific application we are interested in here may be expressed abstractly as 
a kind of combinatorial auction with preference elicitation.
Specifically, we suppose there is a collection of items on a menu, 
and each possible bundle of items has an associated fixed price.  There is a stream of 
customers, each with a valuation function that provides a value for each possible bundle of 
items.  The objective is to serve each customer a bundle of items that nearly-maximizes
his or her surplus value (value minus price).  However, we are not permitted direct observation 
of the customer valuation functions; rather, we may query for the value of any given 
bundle of items; this is referred to as a \emph{value query} in the literature on 
preference elicitation in combinatorial auctions (see Chapter 14 of~\cite{CSS06}, \cite{zinkevich:03}).
The objective is to achieve this near-maximal surplus guarantee, while 
making only a small number of queries per customer.  We suppose the customer valuation function
are sampled i.i.d. according to an unknown distribution over a known (but arbitrary) class of
real-valued functions having finite pseudo-dimension.  Reasoning that knowledge of this 
distribution should allow one to make a smaller number of value queries per customer, 
we are interested in estimating this unknown distribution, so that as we serve more and more
customers, the number of queries per customer required to identify a near-optimal bundle
should decrease.  In this context, we in fact prove that in the limit, the expected number
of queries per customer converges to the number required of a method having direct knowledge
of the true distribution of valuation functions. 

Formally, suppose there is a menu of $n$ items $[n] = \{1,\ldots,n\}$, and each bundle $B \subseteq [n]$ 
has an associated price $p(B) \geq 0$.  Suppose also there is a sequence of customers, each with a 
valuation function $v_{t} : 2^{[n]} \to \reals$.  We suppose these $v_{t}$ functions are i.i.d. samples.
We can then calculate the satisfaction function for each customer as $s_{t}(x)$, where $x \in \{0,1\}^{n}$,
and $s_{t}(x) = v_{t}(B_{x}) - p(B_{x})$, where $B_{x} \subseteq [n]$ contains element $i \in [n]$ iff $x_i = 1$.

Now suppose we are able to ask each customer a number of questions before serving up a bundle $B_{\hat{x}_{t}}$ 
to that customer.  More specifically, we are able to ask for the value $s_{t}(x)$ for any $x \in \{0,1\}^{n}$.  This is referred to as a \emph{value query} in the literature on 
preference elicitation in combinatorial auctions (see Chapter 14 of~\cite{CSS06}, \cite{zinkevich:03}).
We are interested in asking as few questions as possible, while satisfying the 
guarantee that $\E[ s_{t}(\hat{x}_{t}) - \max_{x} s_{t}(x)] \leq \eps$.

Now suppose, for every $\prior$ and $\eps$, we have a method $A(\prior,\eps)$ such that, 
given that $\prior$ is the actual distribution of the $s_{t}$ functions,
$A(\prior,\eps)$ guarantees that the $\hat{x}_{t}$ value it selects has $\E[ \max_{x} s_{t}(x) - s_{t}(\hat{x}_{t}) ] \leq \eps$;
also let $\hat{N}_{t}(\prior,\eps)$ denote the actual (random) number of queries the method $A(\prior,\eps)$ would ask for the $s_{t}$ function,
and let $Q(\prior,\eps) = \E[ \hat{N}_{t}(\prior,\eps) ]$.
We suppose the method never queries any $s_{t}(x)$ value twice for a given $t$, so that its number of queries for any given $t$ is bounded.

Also suppose $\F$ is a VC subgraph class of functions mapping $\X = \{0,1\}^{n}$ into $[-1,1]$ with pseudodimension $\vc$,
and that $\{\prior_{\param} : \param \in \Params\}$ is a known totally bounded family of
distributions over $\F$ such that the $s_{t}$ functions have distribution $\prior_{\TruParam}$ 
for some unknown $\TruParam \in \Params$.
For any $\param \in \Params$ and $\gamma > 0$, let $\Ball(\param,\gamma) = \{\param^{\prime} \in \Params : \| \prior_{\param} - \prior_{\param^{\prime}}\| \leq \gamma\}$.

Suppose, in addition to $A$, we have another method $A^{\prime}(\eps)$ that is not $\prior$-dependent, but still provides the $\eps$-correctness guarantee,
and makes a bounded number of queries (e.g., in the worst case, we could consider querying all $2^n$ points, but in most cases there are more clever $\prior$-independent methods that use far fewer queries, such as $O(1/\eps^2)$).
Consider the following method; 
the quantities $\hat{\param}_{T\TruParam}$, $R(T,\alpha)$, and $\delta(T,\alpha)$ from Theorem~\ref{thm:consistency} 
are here considered with respect $P_{X}$ taken as the uniform distribution on $\{0,1\}^{n}$.

\begin{algorithm}[h!]
\begin{algorithmic}
\FOR {$t  = 1, 2, \ldots,T$}
\STATE Pick points $X_{t1},X_{t2},\ldots,X_{t\vc}$ uniformly at random from $\{0,1\}^{n}$
\IF {$R(t-1,\eps/2) > \eps/8$}
\STATE Run $A^{\prime}(\eps)$
\STATE Take $\hat{x}_{t}$ as the returned value
\ELSE
\STATE Let $\check{\param}_{t \TruParam} \in \Ball\left(\hat{\param}_{(t-1) \TruParam}, R(t-1,\eps/2)\right)$ be such that \\ 
$Q(\prior_{\check{\param}_{t \TruParam}},\eps/4) \leq\!\!\!\!\! \min\limits_{\param \in \Ball\left( \hat{\param}_{(t-1) \TruParam}, R(t-1, \eps/2)\right)} \!\!\!\!\! Q(\prior_{\param},\eps/4) + \frac{1}{t}$
\STATE Run $A(\prior_{\check{\param}_{t \TruParam}},\eps/4)$ and let $\hat{x}_{t}$ be its return value
\ENDIF
\ENDFOR
\caption{An algorithm for sequentially maximizing expected customer satisfaction.}
\label{alg:transfer}
\end{algorithmic}
\end{algorithm}

The following theorem indicates that this method is correct, and furthermore that
the long-run average number of queries is not much worse than that of a method that
has direct knowledge of $\prior_{\TruParam}$.  The proof of this result parallels that
of \cite{yang:13} for the transfer learning setting, but is included here for completeness.

\begin{theorem}
\label{thm:transfer}
For the above method, $\forall t \leq T, \E[ \max_{x} s_{t}(x) - s_{t}(\hat{x}_{t})] \leq \eps$. 
Furthermore, if $S_{T}(\eps)$ is the total number of queries made by the method,
then 
\[
\limsup\limits_{T \to \infty} \frac{\E[S_{T}(\eps)]}{T} \leq Q(\prior_{\TruParam},\eps/4) + \vc.
\]
\end{theorem}
\begin{proof}
By Theorem~\ref{thm:consistency}, for any $t \leq T$, 
if $R(t-1,\eps/2) \leq \eps/8$, 
then with probability at least $1-\eps/2$, 
$\|\prior_{\TruParam} - \prior_{\hat{\param}_{(t-1)\TruParam}}\| \leq R(t-1,\eps/2)$, 
so that a triangle inequality implies 
$\|\prior_{\TruParam} - \prior_{\check{\param}_{t\TruParam}}\| \leq 2R(t-1,\eps/2) \leq \eps/4$.
Thus,
\begin{multline*}
\E\left[ \max_{x} s_{t}(x) - s_{t}(\hat{x}_{t}) \right]  \notag
\\ \leq \eps/2 + 
\E\left[ \E\left[ \max_{x} s_{t}(x) - s_{t}(\hat{x}_{t}) \Big| \check{\param}_{t \TruParam}\right] \ind\left[ \|\prior_{\check{\param}_{t\TruParam}}-\prior_{\TruParam}\| \leq \eps/2\right]\right].
\end{multline*}
For $\param \in \Params$, let $\hat{x}_{t \param}$ denote the point $x$ that would be returned by $A(\prior_{\check{\param}_{t \TruParam}},\eps/4)$
when queries are answered by some $s_{t \param} \sim \prior_{\param}$ instead of $s_{t}$ (and supposing $s_{t} = s_{t \TruParam}$).
If $\| \prior_{\check{\param}_{t \TruParam}} - \prior_{\TruParam}\| \leq \eps/4$, then 
\begin{align*}
&\E\left[ \max_{x} s_{t}(x) - s_{t}(\hat{x}_{t}) \Big| \check{\param}_{t \TruParam}\right]
= \E\left[ \max_{x} s_{t \TruParam}(x) - s_{t \TruParam}(\hat{x}_{t}) \Big| \check{\param}_{t \TruParam}\right]
\\ & \leq \E\left[ \max_{x} s_{t \check{\param}_{t \TruParam}}(x) - s_{t \check{\param}_{t \TruParam}}(\hat{x}_{t \check{\param}_{t \TruParam}}) \Big| \check{\param}_{t \TruParam}\right]
+ \|\prior_{\check{\param}_{t \TruParam}} - \prior_{\TruParam}\|
\leq \eps/4 + \eps/4 = \eps/2.
\end{align*}
Plugging into the above bound, we have
$\E\left[ \max_{x} s_{t}(x) - s_{t}(\hat{x}_{t}) \right] \leq \eps$.

For the result on $S_{T}(\eps)$, first note that $R(t-1,\eps/2) > \eps/8$ only finitely many times (due to $R(t,\alpha) = o(1)$),
so that we can ignore those values of $t$ in the asymptotic calculation (as the number of queries is always bounded), and rely on the correctness guarantee of $A^{\prime}$.
For the remaining values $t$, let $N_{t}$ denote the number of queries made by $A(\prior_{\check{\param}_{t \TruParam}},\eps/4)$.
Then 
\begin{equation*}
\limsup\limits_{T \to\infty} \frac{\E[S_{T}(\eps)]}{T} \leq \vc + \limsup\limits_{T\to\infty} \sum_{t=1}^{T} \frac{\E\left[ N_{t} \right]}{T}.
\end{equation*}
Since 
\begin{align*}
& \lim\limits_{T\to\infty} \frac{1}{T} \sum_{t=1}^{T} \E\left[ N_{t} \ind[ \| \prior_{\hat{\param}_{(t-1)\TruParam}} - \prior_{\TruParam}\| > R(t-1, \eps/2)] \right]
\\ & \leq \lim\limits_{T \to\infty} \frac{1}{T} \sum_{t=1}^{T} 2^{n} \P\left( \|\prior_{\hat{\param}_{(t-1)\TruParam}} - \prior_{\TruParam}\| > R(t-1,\eps/2)\right)
\\ & \leq 2^{n} \lim\limits_{T\to\infty} \frac{1}{T} \sum_{t=1}^{T} \delta(t-1,\eps/2) = 0,
\end{align*}
we have
\begin{equation*}
\limsup\limits_{T\to\infty} \sum_{t=1}^{T} \frac{\E\left[ N_{t} \right]}{T}
= \limsup\limits_{T\to\infty} \frac{1}{T} \sum_{t=1}^{T} \E\Big[ N_{t} \ind[ \|\prior_{\hat{\param}_{(t-1)\TruParam}} - \prior_{\TruParam}\| \leq R(t-1,\eps/2)]\Big].
\end{equation*}
For $t \leq T$, let $N_{t}(\check{\param}_{t \TruParam})$ denote the number of queries $A(\prior_{\check{\param}_{t \TruParam}},\eps/4)$ 
would make if queries were answered with $s_{t \check{\param}_{t \TruParam}}$ instead of $s_t$.
On the event $\| \prior_{\hat{\param}_{(t-1)\TruParam}} - \prior_{\TruParam}\| \leq R(t-1,\eps/2)$, 
we have 
\begin{align*}
\E\left[ N_{t} \Big| \check{\param}_{t\TruParam}\right]
& \leq \E\left[ N_{t}(\check{\param}_{t \TruParam}) \Big| \check{\param}_{t\TruParam}\right] + 2 R(t-1,\eps/2)
\\ & = Q(\prior_{\check{\param}_{t \TruParam}}\!, \eps/4) + 2R(t\!-\!1,\eps/2) 
\leq Q(\prior_{\TruParam},\eps/4) + 2R(t\!-\!1,\eps/2) + 1/t.
\end{align*}
Therefore, 
\begin{align*}
& \limsup\limits_{T\to\infty} \frac{1}{T} \sum_{t=1}^{T} \E\left[ N_{t} \ind[ \|\prior_{\hat{\param}_{(t-1)\TruParam}} - \prior_{\TruParam}\| \leq R(t-1,\eps/2)]\right]
\\ & \leq Q(\prior_{\TruParam},\eps/4) + \limsup\limits_{T\to\infty} \frac{1}{T} \sum_{t=1}^{T} 2R(t-1,\eps/2) + 1/t
= Q(\prior_{\TruParam},\eps/4).
\end{align*}
\qed
\end{proof}

In many cases, this result will even continue to hold with an infinite number of goods ($n=\infty$), 
since Theorem~\ref{thm:consistency} has no dependence on the cardinality of the space $\X$.

\section{Open Problems}

There are several interesting questions that remain open at this time.
Can either the lower bound or upper bound be improved in general?
If, instead of $\vc$ samples per task, we instead use $m \geq \vc$ samples,
how does the minimax risk vary with $m$?
Related to this, what is the optimal value of $m$ to optimize the rate of 
convergence as a function of $mT$, the total number of samples?
More generally, if an estimator is permitted to use $N$ total samples,
taken from however many tasks it wishes, what is the optimal rate 
of convergence as a function of $N$?

\appendix

\section{Proofs for Section~\ref{sec:real-valued}}
\label{app:real-valued}

The proof of Theorem~\ref{thm:consistency} is based on the following sequence of lemmas,
which parallel those used by \citet{yang:13} for establishing the analogous result for consistent
estimation of priors over binary functions.
The last of these lemmas (namely, Lemma~\ref{lem:k-to-d}) requires substantial modifications
to the original argument of \citet{yang:13}; the others use arguments more-directly based on 
those of \citet{yang:13}.

\begin{lemma}
\label{lem:prior-to-infty}
For any $\param,\param^{\prime} \in \Params$ and $t \in \nats$,
\[
\| \prior_{\param} - \prior_{\param^{\prime}} \|
= \| \P_{\Data^{t}(\param)} - \P_{\Data^{t}(\param^{\prime})} \|.
\]
\end{lemma}
\begin{proof}
Fix $\param, \param^{\prime} \in \Params$, $t \in \nats$.
Let $\DataX = \{X_{t1},X_{t2},\ldots\}$, $\DataY(\param) = \{Y_{t1}(\param), Y_{t2}(\param), \ldots\}$,
and for $k \in \nats$ let $\DataX_{k} = \{X_{t1},\ldots,X_{tk}\}$.
and $\DataY_{k}(\param) = \{Y_{t1}(\param),\ldots,Y_{tk}(\param)\}$.
For $h \in \F$, let $c_{\DataX}(h) = \{(X_{t1},h(X_{t1})),(X_{t2},h(X_{t2})),$ $\ldots\}$.

For $h,g \in \F$, define $\rho_{\DataX}(h,g) 
= \lim\limits_{m \to \infty} \frac{1}{m} \sum_{i=1}^{m} | h(X_{t i}) - g(X_{t i})|$ (if the limit exists),
and $\rho_{\DataX_{k}}(h,g) = \frac{1}{k} \sum_{i=1}^{k} |h(X_{t i}) - g(X_{t i})|$.
Note that since $\F$ is a uniformly bounded VC subgraph class, so is the collection of functions 
$\{ | h - g| : h,g \in \F\}$,
so that the uniform strong law of large numbers implies that with probability one,
$\forall h,g \in \F$, $\rho_{\DataX}(h,g)$ exists and has $\rho_{\DataX}(h,g) = \rho(h,g)$~\cite{vapnik:71}.

Consider any $\param,\param^{\prime} \in \Params$, and any $A \in \Borel_{\F}$.
Then any $h \notin A$ has $\forall g \in A$, $\rho(h,g) > 0$ (by the metric assumption).
Thus, if $\rho_{\DataX}(h,g) = \rho(h,g)$ for all $h,g \in \F$, then
$\forall h \notin A$, 
\begin{align*}
\forall g \in A,\rho_{\DataX}(h,g) = \rho(h,g) > 0 & \implies 
\\ \forall g \in A, c_{\DataX}(h) \neq c_{\DataX}(g) & \implies c_{\DataX}(h) \notin c_{\DataX}(A).
\end{align*}
This implies $c_{\DataX}^{-1}(c_{\DataX}(A)) = A$.
Under these conditions,
\[\P_{\Data^{t}(\param) | \DataX}(c_{\DataX}(A)) = \prior_{\param}(c_{\DataX}^{-1}(c_{\DataX}(A))) = \prior_{\param}(A),\]
and similarly for $\param^{\prime}$.

Any measurable set $C$ for the range of $\Data^{t}(\param)$ can be expressed as 
$C = \{c_{\bar{x}}(h) : (h, \bar{x}) \in C^{\prime}\}$ for some appropriate $C^{\prime} \in \Borel_{\F} \otimes \BorelX^{\infty}$.
Letting $C^{\prime}_{\bar{x}} = \{ h : (h,\bar{x}) \in C^{\prime}\}$,
we have 
\begin{equation*}
\P_{\Data^{t}(\param)}(C) 
= \int \prior_{\param}(c_{\bar{x}}^{-1}(c_{\bar{x}}(C^{\prime}_{\bar{x}}))) \P_{\DataX}({\rm d}\bar{x})
= \int \prior_{\param}(C^{\prime}_{\bar{x}}) \P_{\DataX}({\rm d}\bar{x}) 
= \P_{(\target_{t \param},\DataX)}(C^{\prime}).
\end{equation*}
Likewise, this reasoning holds for $\param^{\prime}$.
Then 
\begin{align*}
\| \P_{\Data^{t}(\param)} - \P_{\Data^{t}(\param^{\prime})} \|
& = \| \P_{(\target_{t \param}, \DataX)} - \P_{(\target_{t \param^{\prime}},\DataX)} \|
\\ & = \sup_{C^{\prime} \in \Borel_{\F} \otimes \BorelX^{\infty}} \left| \int (\prior_{\param}(C^{\prime}_{\bar{x}})- \prior_{\param^{\prime}}(C^{\prime}_{\bar{x}})) \P_{\DataX}({\rm d}\bar{x}) \right|
\\ & \leq \int \sup_{A \in \Borel_{\F}} | \prior_{\param}(A) - \prior_{\param^{\prime}}(A) | \P_{\DataX}({\rm d}\bar{x})
= \|\prior_{\param} - \prior_{\param^{\prime}}\|.
\end{align*}
Since $\target_{t \param}$ and $\DataX$ are independent, $\forall A \in \Borel_{\F}$,
$\prior_{\param}(A) = \P_{\target_{t \param}}(A) = \P_{\target_{t \param}}(A) \P_{\DataX}(\X^{\infty})$ $= \P_{(\target_{t \param},\DataX)}(A \times \X^{\infty})$.
Analogous reasoning holds for $\target_{t \param^{\prime}}$.
Thus, we have
\begin{align*}
\|\prior_{\param} - \prior_{\param^{\prime}}\| 
& = \| \P_{(\target_{t \param}, \DataX)}(\cdot \times \X^{\infty}) - \P_{(\target_{t \param^{\prime}},\DataX)}(\cdot \times \X^{\infty})\| 
\\ & \leq \| \P_{(\target_{t \param}, \DataX)} - \P_{(\target_{t \param^{\prime}},\DataX)} \|
=  \| \P_{\Data^{t}(\param)} - \P_{\Data^{t}(\param^{\prime})} \|.
\end{align*}
Altogether, we have $\| \P_{\Data^{t}(\param)} - \P_{\Data^{t}(\param^{\prime})} \| = \|\prior_{\param} - \prior_{\param^{\prime}}\|$.
\qed
\end{proof}

\begin{lemma}
\label{lem:infty-to-k}
There exists a sequence $r_k = o(1)$ such that, $\forall t,k \in \nats$, $\forall \param,\param^{\prime} \in \Params$,
\[
\| \P_{\Data^{t}_{k}(\param)} - \P_{\Data^{t}_{k}(\param^{\prime})} \| 
\leq \| \prior_{\param} - \prior_{\param^{\prime}} \|
\leq \| \P_{\Data^{t}_{k}(\param)} - \P_{\Data^{t}_{k}(\param^{\prime})} \| + r_k.
\]
\end{lemma}
\begin{proof}
This proof follows identically to a proof of \cite{yang:13}, but is included here for completeness.
Since $\P_{\Data^{t}_{k}(\param)}(A) = \P_{\Data^{t}(\param)}(A \times (\X \times \reals)^{\infty})$ for all measurable $A \subseteq (\X \times \reals)^{k}$,
and similarly for $\param^{\prime}$,
we have
\begin{align*}
& \| \P_{\Data^{t}_{k}(\param)} - \P_{\Data^{t}_{k}(\param^{\prime})} \|
= \sup_{A \in \Borel^{k}} \P_{\Data^{t}_{k}(\param)}(A) - \P_{\Data^{t}_{k}(\param^{\prime})}(A)
\\ & = \sup_{A \in \Borel^{k}} \P_{\Data^{t}(\param)}(A \times (\X \times \reals)^{\infty}) - \P_{\Data^{t}(\param^{\prime})}(A \times (\X \times \reals)^{\infty})
\\ & \leq \sup_{A \in \Borel^{\infty}} \P_{\Data^{t}(\param)}(A) - \P_{\Data^{t}(\param^{\prime})}(A)
= \| \P_{\Data^{t}(\param)} - \P_{\Data^{t}(\param^{\prime})} \|,
\end{align*}
which implies the left inequality when combined with Lemma~\ref{lem:prior-to-infty}.

Next, we focus on the right inequality.
Fix $\param,\param^{\prime} \in \Params$ and $\gamma > 0$,
and let $B \in \Borel^{\infty}$ be such that
\[
\| \prior_{\param} - \prior_{\param^{\prime}} \|
= \| \P_{\Data^{t}(\param)} - \P_{\Data^{t}(\param^{\prime})} \|
< \P_{\Data^{t}(\param)}(B) - \P_{\Data^{t}(\param^{\prime})}(B) + \gamma.
\]
Let $\A = \{ A \times (\X\times\reals)^{\infty} : A \in \Borel^{k}, k \in \nats\}$.
Note that $\A$ is an algebra that generates $\Borel^{\infty}$.
Thus, Carath\'{e}odory's extension theorem (specifically, the 
version presented by \cite{schervish:95})
implies that there exist disjoint sets $\{A_i\}_{i\in\nats}$ in $\A$
such that $B \subseteq \bigcup_{i\in\nats} A_i$ and
\begin{equation*}
\P_{\Data^{t}(\param)}(B) - \P_{\Data^{t}(\param^{\prime})}(B) 
< \sum_{i \in \nats} \P_{\Data^{t}(\param)}(A_i) - \sum_{i \in \nats} \P_{\Data^{t}(\param^{\prime})}(A_i) + \gamma.
\end{equation*}
Since these $A_i$ sets are disjoint, each of these sums is bounded by a probability value, 
which implies that there exists some $n \in \nats$ such that
\[
\sum_{i \in \nats} \P_{\Data^{t}(\param)}(A_i) < \gamma + \sum_{i=1}^{n} \P_{\Data^{t}(\param)}(A_i),
\]
which implies
\begin{align*}
\sum_{i\in\nats} \P_{\Data^{t}(\param)}(A_i) - \sum_{i\in\nats} \P_{\Data^{t}(\param^{\prime})}(A_i) 
& < \gamma + \sum_{i=1}^{n} \P_{\Data^{t}(\param)}(A_i) - \sum_{i=1}^{n} \P_{\Data^{t}(\param^{\prime})}(A_i)
\\ & = \gamma + \P_{\Data^{t}(\param)}\left( \bigcup_{i=1}^{n} A_i \right) - \P_{\Data^{t}(\param^{\prime})}\left( \bigcup_{i=1}^{n} A_i \right).
\end{align*}
As $\bigcup_{i=1}^{n} A_i \in \A$, there exists $m \in \nats$ and measurable $B_{m} \in \Borel^{m}$ such that
$\bigcup_{i=1}^{n} A_i = B_{m} \times (\X \times \reals)^{\infty}$, and therefore
\begin{multline*}
\P_{\Data^{t}(\param)}\left( \bigcup_{i=1}^{n} A_i \right) - \P_{\Data^{t}(\param^{\prime})}\left(\bigcup_{i=1}^{n} A_i\right) 
= \P_{\Data^{t}_{m}(\param)}(B_{m}) - \P_{\Data^{t}_{m}(\param^{\prime})}(B_{m})
\\ \leq \| \P_{\Data^{t}_{m}(\param)} - \P_{\Data^{t}_{m}(\param^{\prime})} \|
\leq \lim_{k \to \infty} \| \P_{\Data^{t}_{k}(\param)} - \P_{\Data^{t}_{k}(\param^{\prime})} \|.
\end{multline*}
Combining the above, we have
$\| \prior_{\param} - \prior_{\param^{\prime}} \| \leq \lim_{k\to\infty} \| \P_{\Data^{t}_{k}(\param)} - \P_{\Data^{t}_{k}(\param^{\prime})} \| + 3\gamma$.
By letting $\gamma$ approach $0$, we have
\[
\| \prior_{\param} - \prior_{\param^{\prime}} \| \leq \lim_{k\to\infty} \| \P_{\Data^{t}_{k}(\param)} - \P_{\Data^{t}_{k}(\param^{\prime})} \|.
\]
So there exists a sequence $r_{k}(\param,\param^{\prime}) = o(1)$ such that
\[
\forall k \in \nats, \|\prior_{\param} - \prior_{\param^{\prime}} \| \leq \| \P_{\Data^{t}_{k}(\param)} - \P_{\Data^{t}_{k}(\param^{\prime})} \| + r_{k}(\param,\param^{\prime}).
\]
Now let $\gamma > 0$ and let $\Params_{\gamma}$ be a minimal $\gamma$-cover of $\Params$.
Define the quantity $r_{k}(\gamma) = \max_{\param,\param^{\prime} \in \Params_{\gamma}} r_{k}(\param,\param^{\prime})$.
Then for any $\param,\param^{\prime} \in \Params$, let 
$\param_{\gamma} = \argmin_{\param^{\prime\prime} \in \Params_{\gamma}} \|\prior_{\param} - \prior_{\param^{\prime\prime}} \|$
and
$\param_{\gamma}^{\prime} = \argmin_{\param^{\prime\prime} \in \Params_{\gamma}} \|\prior_{\param^{\prime}} - \prior_{\param^{\prime\prime}} \|$.
Then a triangle inequality implies that $\forall k \in \nats$,
\begin{align*}
& \| \prior_{\param} - \prior_{\param^{\prime}} \| 
\leq \|\prior_{\param} - \prior_{\param_{\gamma}} \| + \|\prior_{\param_{\gamma}} - \prior_{\param_{\gamma}^{\prime}}\| + \|\prior_{\param_{\gamma}^{\prime}} - \prior_{\param^{\prime}}\|
\\ & < 2 \gamma + r_{k}(\param_{\gamma},\param_{\gamma}^{\prime}) + \| \P_{\Data^{t}_{k}(\param_{\gamma})} - \P_{\Data^{t}_{k}(\param_{\gamma}^{\prime})} \|
\leq 2 \gamma + r_{k}(\gamma) + \| \P_{\Data^{t}_{k}(\param_{\gamma})} - \P_{\Data^{t}_{k}(\param_{\gamma}^{\prime})} \|.
\end{align*}
Triangle inequalities and the left inequality from the lemma statement (already established) imply
\begin{align*}
& \|\P_{\!\Data^{t}_{k}(\param_{\gamma})} \!-\! \P_{\!\Data^{t}_{k}(\param_{\gamma}^{\prime})}\|
\leq \|\P_{\!\Data^{t}_{k}(\param_{\gamma})} \!-\! \P_{\!\Data^{t}_{k}(\param)}\| 
+ \|\P_{\!\Data^{t}_{k}(\param)} \!-\! \P_{\!\Data^{t}_{k}(\param^{\prime})}\| 
+ \|\P_{\!\Data^{t}_{k}(\param_{\gamma}^{\prime})} \!-\! \P_{\!\Data^{t}_{k}(\param^{\prime})}\|
\\ & \leq \|\prior_{\param_{\gamma}} - \prior_{\param}\| + \|\P_{\Data^{t}_{k}(\param)} - \P_{\Data^{t}_{k}(\param^{\prime})}\| + \|\prior_{\param_{\gamma}^{\prime}} - \prior_{\param^{\prime}}\|
< 2\gamma + \|\P_{\Data^{t}_{k}(\param)} - \P_{\Data^{t}_{k}(\param^{\prime})}\|.
\end{align*}
So in total we have
\[
\| \prior_{\param} - \prior_{\param^{\prime}} \|
\leq 4 \gamma + r_{k}(\gamma) + \|\P_{\Data^{t}_{k}(\param)} - \P_{\Data^{t}_{k}(\param^{\prime})}\|.
\]
Since this holds for all $\gamma > 0$, defining $r_{k} = \inf_{\gamma > 0} (4 \gamma + r_{k}(\gamma))$,
we have the right inequality of the lemma statement.  Furthermore, since each $r_{k}(\param,\param^{\prime}) = o(1)$,
and $|\Params_{\gamma}| < \infty$, we have $r_{k}(\gamma) = o(1)$ for each $\gamma > 0$,
and thus we also have $r_{k} = o(1)$.
\qed
\end{proof}

\begin{lemma}
\label{lem:k-to-d}
$\forall t,k \in \nats$, there exists a monotone function $M_{k}(x) = o(1)$ such that,
$\forall \param,\param^{\prime} \in \Params$,
\[
\| \P_{\Data^{t}_{k}(\param)} - \P_{\Data^{t}_{k}(\param^{\prime})}\| \leq M_{k}\left(\| \P_{\Data^{t}_{\vc}(\param)} - \P_{\Data^{t}_{\vc}(\param^{\prime})} \|\right).
\]
\end{lemma}
\begin{proof}
Fix any $t \in \nats$, and let
$\DataX = \{X_{t1},X_{t2},\ldots\}$ and $\DataY(\param) = \{Y_{t1}(\param), Y_{t2}(\param), \ldots\}$,
and for $k \in \nats$ let $\DataX_{k} = \{X_{t1},\ldots,X_{tk}\}$
and $\DataY_{k}(\param) = \{Y_{t1}(\param),\ldots,Y_{tk}(\param)\}$.

If $k \leq \vc$, then 
$\P_{\Data^{t}_{k}(\param)}(\cdot) = \P_{\Data^{t}_{\vc}(\param)}(\cdot \times (\X \times \{-1,+1\})^{\vc-k})$,
so that 
\[\|\P_{\Data^{t}_{k}(\param)} - \P_{\Data^{t}_{k}(\param^{\prime})}\| \leq \|\P_{\Data^{t}_{\vc}(\param)} - \P_{\Data^{t}_{\vc}(\param^{\prime})}\|,\]
and therefore the result trivially holds.  

Now suppose $k > \vc$.  Fix any $\gamma > 0$, and let $B_{\param,\param^{\prime}} \subseteq (\X \times \reals)^{k}$ be 
a measurable set such that 
\begin{align*}
\P_{\Data^{t}_{k}(\param)}(B_{\param,\param^{\prime}}) - \P_{\Data^{t}_{k}(\param^{\prime})}(B_{\param,\param^{\prime}}) 
& \leq \| \P_{\Data^{t}_{k}(\param)} - \P_{\Data^{t}_{k}(\param^{\prime})}\| 
\\ & \leq \P_{\Data^{t}_{k}(\param)}(B_{\param,\param^{\prime}}) - \P_{\Data^{t}_{k}(\param^{\prime})}(B_{\param,\param^{\prime}}) + \gamma.
\end{align*}
By Carath\'{e}odory's extension theorem (specifically, the version presented by \cite{schervish:95}), 
there exists a disjoint sequence of sets $\{B_i(\param,\param^{\prime})\}_{i=1}^{\infty}$ such that 
\begin{equation*}
\P_{\Data^{t}_{k}(\param)}(B_{\param,\param^{\prime}}) - \P_{\Data^{t}_{k}(\param^{\prime})}(B_{\param,\param^{\prime}}) 
< \gamma + \sum_{i=1}^{\infty} \P_{\Data^{t}_{k}(\param)}(B_i(\param,\param^{\prime})) - \sum_{i=1}^{\infty} \P_{\Data^{t}_{k}(\param^{\prime})}(B_i(\param,\param^{\prime})),
\end{equation*}
and such that each $B_i(\param,\param^{\prime})$ is representable as follows;
for some $\ell_i(\param,\param^{\prime}) \in \nats$, and sets $C_{i j} = (A_{i j 1} \times (-\infty,t_{i j 1}])\times \cdots \times (A_{i j k} \times (-\infty,t_{i j k}])$, for $j \leq \ell_i(\param,\param^{\prime})$,
where each $A_{i j p} \in \BorelX$,
the set $B_i(\param,\param^{\prime})$ is representable as $\bigcup_{s \in S_i} \bigcap_{j=1}^{\ell_i(\param,\param^{\prime})} D_{i j s}$,
where $S_i \subseteq \{0,\ldots,2^{\ell_i(\param,\param^{\prime})}-1\}$, each $D_{i j s} \in \{C_{i j}, C_{i j}^{c}\}$, 
and $s \neq s^{\prime} \Rightarrow \bigcap_{j=1}^{\ell_i(\param,\param^{\prime})} D_{i j s} \cap \bigcap_{j=1}^{\ell_i(\param,\param^{\prime})} D_{i j s^{\prime}} = \emptyset$.
Since the $B_i(\param,\param^{\prime})$ are disjoint, the above sums are bounded, so that there exists
$m_{k}(\param,\param^{\prime},\gamma) \in \nats$ such that every $m \geq m_{k}(\param,\param^{\prime},\gamma)$ has
\begin{equation*}
\P_{\!\Data^{t}_{k}(\param)}(B_{\param,\param^{\prime}}) - \P_{\!\Data^{t}_{k}(\param^{\prime})}(B_{\param,\param^{\prime}}) 
< 2\gamma + \sum_{i=1}^{m} \P_{\!\Data^{t}_{k}(\param)}(B_i(\param,\param^{\prime})) - \sum_{i=1}^{m} \P_{\!\Data^{t}_{k}(\param^{\prime})}(B_i(\param,\param^{\prime})),
\end{equation*}
Now define $\tilde{M}_{k}(\gamma) = \max_{\param,\param^{\prime} \in \Params_{\gamma}} m_{k}(\param,\param^{\prime},\gamma)$.
Then for any $\param,\param^{\prime} \in \Params$, let $\param_{\gamma},\param_{\gamma}^{\prime} \in \Params_{\gamma}$ be 
such that $\|\prior_{\param} - \prior_{\param_{\gamma}}\| < \gamma$ and $\|\prior_{\param^{\prime}}-\prior_{\param^{\prime}_{\gamma}}\| < \gamma$,
which implies $\|\P_{\Data^{t}_{k}(\param)} - \P_{\Data^{t}_{k}(\param_{\gamma})}\| < \gamma$ and $\|\P_{\Data^{t}_{k}(\param^{\prime})} - \P_{\Data^{t}_{k}(\param^{\prime}_{\gamma})}\| < \gamma$
by Lemma~\ref{lem:infty-to-k}.
Then
\begin{align*}
\|\P_{\Data^{t}_{k}(\param)} - \P_{\Data^{t}_{k}(\param^{\prime})}\| 
& < \|\P_{\Data^{t}_{k}(\param_{\gamma})} - \P_{\Data^{t}_{k}(\param^{\prime}_{\gamma})}\| + 2\gamma
\\ & \leq \P_{\Data^{t}_{k}(\param_{\gamma})}(B_{\param_{\gamma},\param^{\prime}_{\gamma}}) - \P_{\Data^{t}_{k}(\param_{\gamma}^{\prime})}(B_{\param_{\gamma},\param^{\prime}_{\gamma}}) + 3\gamma
\\ & \leq \sum_{i=1}^{\tilde{M}_{k}(\gamma)} \P_{\Data^{t}_{k}(\param_{\gamma})}(B_i(\param_{\gamma},\param^{\prime}_{\gamma})) - \P_{\Data^{t}_{k}(\param^{\prime}_{\gamma})}(B_i(\param_{\gamma},\param^{\prime}_{\gamma})) + 5\gamma.
\end{align*}
Again, since the $B_i(\param_{\gamma},\param^{\prime}_{\gamma})$ are disjoint, this equals
\begin{align*}
& 5\gamma + \P_{\Data^{t}_{k}(\param_{\gamma})}\left( \bigcup_{i=1}^{\tilde{M}_{k}(\gamma)} B_i(\param_{\gamma},\param^{\prime}_{\gamma}) \right) 
- \P_{\Data^{t}_{k}(\param^{\prime}_{\gamma})}\left( \bigcup_{i=1}^{\tilde{M}_{k}(\gamma)} B_i(\param_{\gamma},\param^{\prime}_{\gamma}) \right)
\\ & \leq 7\gamma + \P_{\Data^{t}_{k}(\param)}\left( \bigcup_{i=1}^{\tilde{M}_{k}(\gamma)} B_i(\param_{\gamma},\param^{\prime}_{\gamma}) \right) 
- \P_{\Data^{t}_{k}(\param^{\prime})}\left( \bigcup_{i=1}^{\tilde{M}_{k}(\gamma)} B_i(\param_{\gamma},\param^{\prime}_{\gamma}) \right)
\\ & = 7\gamma + \sum_{i=1}^{\tilde{M}_{k}(\gamma)} \P_{\Data^{t}_{k}(\param)}(B_i(\param_{\gamma},\param^{\prime}_{\gamma})) - \P_{\Data^{t}_{k}(\param^{\prime})}(B_i(\param_{\gamma},\param^{\prime}_{\gamma}))
\\ & \leq 7 \gamma + 
\tilde{M}_{k}(\gamma) \max_{i \leq \tilde{M}_{k}(\gamma)} \left| \P_{\Data^{t}_{k}(\param)}(B_i(\param_{\gamma},\param^{\prime}_{\gamma})) - \P_{\Data^{t}_{k}(\param^{\prime})}(B_i(\param_{\gamma},\param^{\prime}_{\gamma})) \right|.
\end{align*}
Thus, if we can show that each term $\left| \P_{\Data^{t}_{k}(\param)}(B_i(\param_{\gamma},\param^{\prime}_{\gamma})) - \P_{\Data^{t}_{k}(\param^{\prime})}(B_i(\param_{\gamma},\param^{\prime}_{\gamma})) \right|$
is bounded by a $o(1)$ function of $\| \P_{\Data^{t}_{\vc}(\param)} - \P_{\Data^{t}_{\vc}(\param^{\prime})}\|$, then the result will follow by substituting this relaxation into the above expression
and defining $M_{k}$ by minimizing the resulting expression over $\gamma > 0$.

Toward this end, let $C_{ij}$ be as above from the definition of $B_i(\param_{\gamma},\param^{\prime}_{\gamma})$,
and note that $I_{B_i(\param_{\gamma},\param^{\prime}_{\gamma})}$ is representable as a function of the $I_{C_{ij}}$ indicators,
so that 
\begin{align*}
&\left| \P_{\!\Data^{t}_{k}(\param)}(B_i(\param_{\gamma},\param^{\prime}_{\gamma})) \!-\! \P_{\!\Data^{t}_{k}(\param^{\prime})}(B_i(\param_{\gamma},\param^{\prime}_{\gamma})) \right|
= \| \P_{\!I_{B_i(\param_{\gamma},\param^{\prime}_{\gamma})}(\Data^{t}_{k}(\param))} \!-\! \P_{\!I_{B_i(\param_{\gamma},\param^{\prime}_{\gamma})}(\Data^{t}_{k}(\param^{\prime}))}\|
\\ & \leq \| \P_{ (I_{C_{i1}}(\Data^{t}_{k}(\param)),\ldots,I_{C_{i \ell_{i}(\param_{\gamma},\param^{\prime}_{\gamma})}}(\Data^{t}_{k}(\param)))} 
- \P_{ (I_{C_{i1}}(\Data^{t}_{k}(\param^{\prime})),\ldots,I_{C_{i \ell_{i}(\param_{\gamma},\param^{\prime}_{\gamma})}}(\Data^{t}_{k}(\param^{\prime})))} \|
\\ & \leq 2^{\ell_{i}(\param_{\gamma},\param^{\prime}_{\gamma})} \max_{J \subseteq \{1,\ldots,\ell_{i}(\param_{\gamma},\param^{\prime}_{\gamma})\}} 
\E\Bigg[ \Bigg( \prod_{j \in J} I_{C_{ij}}(\Data^{t}_{k}(\param))\Bigg) \prod_{j \notin J} \Bigg(1 - I_{C_{ij}}(\Data^{t}_{k}(\param))\Bigg) 
\\ & {\hskip 4.6cm}- \Bigg( \prod_{j \in J} I_{C_{ij}}(\Data^{t}_{k}(\param^{\prime}))\Bigg) \prod_{j \notin J} \Bigg(1 - I_{C_{ij}}(\Data^{t}_{k}(\param^{\prime}))\Bigg) \Bigg]
\\ & \leq 2^{\ell_{i}(\param_{\gamma},\param^{\prime}_{\gamma})} \sum_{J \subseteq \left\{1,\ldots,2^{\ell_{i}(\param_{\gamma},\param^{\prime}_{\gamma})}\right\}} 
\left|\E\left[ \prod_{j \in J} I_{C_{ij}}(\Data^{t}_{k}(\param)) - \prod_{j \in J} I_{C_{ij}}(\Data^{t}_{k}(\param^{\prime}))\right] \right|
\\ & \leq 4^{\ell_{i}(\param_{\gamma},\param^{\prime}_{\gamma})} \max_{J \subseteq \left\{1,\ldots,2^{\ell_{i}(\param_{\gamma},\param^{\prime}_{\gamma})}\right\}} 
\left|\E\left[ \prod_{j \in J} I_{C_{ij}}(\Data^{t}_{k}(\param)) - \prod_{j \in J} I_{C_{ij}}(\Data^{t}_{k}(\param^{\prime}))\right] \right|
\\ & = 4^{\ell_{i}(\param_{\gamma},\param^{\prime}_{\gamma})} \max_{J \subseteq \left\{1,\ldots,2^{\ell_{i}(\param_{\gamma},\param^{\prime}_{\gamma})}\right\}} 
\left|\P_{\Data^{t}_{k}(\param)}\left( \bigcap_{j \in J} C_{ij}\right) - \P_{\Data^{t}_{k}(\param^{\prime})}\left( \bigcap_{j \in J} C_{ij} \right) \right|.
\end{align*}
Note that $\bigcap_{j \in J} C_{ij}$ can be expressed as some 
$(A_{1} \!\times (\!-\infty, t_{1}]) \times \cdots \times (A_{k} \!\times (\!-\infty,t_{k}])$, where each $A_p \in \BorelX$ and $t_p \in \reals$,
so that,
for $\hat{\ell} = \max_{\param,\param^{\prime} \in \Params_{\gamma}} \max_{i \leq \tilde{M}_{k}(\gamma)} \ell_i(\param,\param^{\prime})$
and $\mathcal{C}_{k} = \{ (A_1 \times (-\infty,t_1]) \times \cdots \times (A_k \times (-\infty,t_k]) : \forall j \leq k, A_j \in \BorelX, t_k \in \reals\}$,
this last expression is at most
\[
4^{\hat{\ell}} \sup_{C \in \mathcal{C}_{k}} \left|\P_{\Data^{t}_{k}(\param)}(C) - \P_{\Data^{t}_{k}(\param^{\prime})}(C) \right|.
\]
Next note that for any $C = (A_1 \times (-\infty,t_1]) \times \cdots \times (A_k \times (-\infty,t_k]) \in \mathcal{C}_{k}$, 
letting $C_1 = A_1 \times \cdots \times A_k$ and $C_2 = (-\infty,t_1] \times \cdots \times (-\infty,t_k]$,
\begin{align*}
\P_{\Data^{t}_{k}(\param)}(C) - \P_{\Data^{t}_{k}(\param^{\prime})}(C)
& = \E\left[ \left(\P_{\DataY_{t k}(\param) | \DataX_{t k}}(C_2) - \P_{\DataY_{t k}(\param^{\prime}) | \DataX_{t k}}(C_2)\right) I_{C_1}(\DataX_{t k}) \right]
\\ & \leq \E\left[ \left|\P_{\DataY_{t k}(\param) | \DataX_{t k}}(C_2) - \P_{\DataY_{t k}(\param^{\prime}) | \DataX_{t k}}(C_2) \right|\right]. 
\end{align*}
For $p \in \{1,\ldots,k\}$, let $C_{2 p} = (-\infty, t_p]$.
Then note that, by definition of $\vc$, for any given $x = (x_1,\ldots,x_k)$, 
the class $\H_{x} = \{ x_p \mapsto I_{C_{2 p}}(h(x_p)) : h \in \F\}$ is a VC class over $\{x_1,\ldots,x_k\}$ with VC dimension at most $\vc$.
Furthremore, we have
\begin{multline*}
\left|\P_{\DataY_{t k}(\param) | \DataX_{t k}}(C_2) - \P_{\DataY_{t k}(\param^{\prime}) | \DataX_{t k}}(C_2) \right|
\\ = \Big| \P_{ ( I_{C_{2 1}}(\target_{t \param}(X_{t 1})),\ldots,I_{C_{2 k}}(\target_{t \param}(X_{t k}))) | \DataX_{t k}}( \{(1,\ldots,1)\} ) 
\\ -\P_{ ( I_{C_{2 1}}(\target_{t \param^{\prime}}(X_{t 1})),\ldots,I_{C_{2 k}}(\target_{t \param^{\prime}}(X_{t k}))) | \DataX_{t k}}( \{(1,\ldots,1)\}) \Big|.
\end{multline*}
Therefore, the results of \cite{yang:13} (in the proof of their Lemma 3) imply
that 
\begin{align*}
&\left|\P_{\DataY_{t k}(\param) | \DataX_{t k}}(C_2) - \P_{\DataY_{t k}(\param^{\prime}) | \DataX_{t k}}(C_2) \right|
\\ & \leq 2^{k} \max_{y \in \{0,1\}^{\vc}} \max_{D \in \{1,\ldots,k\}^{\vc}} 
\Big| \P_{ \{I_{C_{2 j}}(\target_{t \param}(X_{t j}))\}_{j \in D} | \{X_{t j}\}_{j \in D}}( \{y\} ) 
\\ & {\hskip 5cm}- \P_{ \{I_{C_{2 j}}(\target_{t \param^{\prime}}(X_{t j}))\}_{j \in D} | \{X_{t j}\}_{j \in D}}( \{y\} ) \Big|.
\end{align*}
Thus, we have
\begin{align*}
& \E\left[ \left|\P_{\DataY_{t k}(\param) | \DataX_{t k}}(C_2) - \P_{\DataY_{t k}(\param^{\prime}) | \DataX_{t k}}(C_2) \right|\right]
\\ & \leq 2^{k} \E\Bigg[ \max_{y \in \{0,1\}^{\vc}} \max_{D \in \{1,\ldots,k\}^{\vc}} 
\Big| \P_{ \{I_{C_{2 j}}(\target_{t \param}(X_{t j}))\}_{j \in D} | \{X_{t j}\}_{j \in D}}( \{y\} ) 
\\ & {\hskip 5cm} - \P_{ \{I_{C_{2 j}}(\target_{t \param^{\prime}}(X_{t j}))\}_{j \in D} | \{X_{t j}\}_{j \in D}}( \{y\} ) \Big|\Bigg]
\\ & 
\leq 2^{k} \sum_{y \in \{0,1\}^{\vc}} \sum_{D \in \{1,\ldots,k\}^{\vc}} 
\E\Bigg[
\Big| \P_{ \{I_{C_{2 j}}(\target_{t \param}(X_{t j}))\}_{j \in D} | \{X_{t j}\}_{j \in D}}( \{y\} ) 
\\ & {\hskip 5cm}- \P_{ \{I_{C_{2 j}}(\target_{t \param^{\prime}}(X_{t j}))\}_{j \in D} | \{X_{t j}\}_{j \in D}}( \{y\} ) \Big|\Bigg]
\\ & \leq 2^{\vc+k} k^{\vc} \max_{y \in \{0,1\}^{\vc}} \max_{D \in \{1,\ldots,k\}^{\vc}} \E\Bigg[
\Big| \P_{ \{I_{C_{2 j}}(\target_{t \param}(X_{t j}))\}_{j \in D} | \{X_{t j}\}_{j \in D}}( \{y\} ) 
\\ & {\hskip 5cm}- \P_{ \{I_{C_{2 j}}(\target_{t \param^{\prime}}(X_{t j}))\}_{j \in D} | \{X_{t j}\}_{j \in D}}( \{y\} ) \Big|\Bigg].
\end{align*}
Exchangeability implies this is at most
\begin{align*}
& 2^{\vc+k} k^{\vc} \max_{y \in \{0,1\}^{\vc}} \sup_{t_1,\ldots,t_{\vc} \in \reals} \E\Bigg[
\Big| \P_{ \{I_{(-\infty,t_j]}(\target_{t \param}(X_{t j}))\}_{j=1}^{\vc} | \DataX_{t \vc}}( \{y\} ) 
\\ & {\hskip 5.6cm}- \P_{ \{I_{(-\infty,t_j]}(\target_{t \param^{\prime}}(X_{t j}))\}_{j=1}^{\vc}  | \DataX_{t \vc}}( \{y\} ) \Big|\Bigg]
\\ & = 2^{\vc+k} k^{\vc} \max_{y \in \{0,1\}^{\vc}} \sup_{t_1,\ldots,t_{\vc} \in \reals} \E\Bigg[
\Big| \P_{ \{I_{(-\infty,t_j]}(Y_{t j}(\param))\}_{j=1}^{\vc} | \DataX_{t \vc}}( \{y\} ) 
\\ & {\hskip 5.6cm}- \P_{ \{I_{(-\infty,t_j]}(Y_{t j}(\param^{\prime}))\}_{j=1}^{\vc}  | \DataX_{t \vc}}( \{y\} ) \Big|\Bigg].
\end{align*}
\cite{yang:13} argue that for all $y \in \{0,1\}^{\vc}$ and $t_1,\ldots,t_{\vc} \in \reals$, 
\begin{align*}
& \E\Big[\Big| \P_{ \{I_{(-\infty,t_j]}(Y_{t j}(\param))\}_{j=1}^{\vc} | \DataX_{t \vc}}( \{y\} ) 
- \P_{ \{I_{(-\infty,t_j]}(Y_{t j}(\param^{\prime}))\}_{j=1}^{\vc}  | \DataX_{t \vc}}( \{y\} ) \Big|\Big]
\\ & \leq 4 \sqrt{\| \P_{ \{I_{(-\infty,t_j]}(Y_{t j}(\param))\}_{j=1}^{\vc}, \DataX_{t \vc}} - \P_{ \{I_{(-\infty,t_j]}(Y_{t j}(\param^{\prime}))\}_{j=1}^{\vc}, \DataX_{t \vc}} \|}.
\end{align*}
Noting that 
\begin{equation*}
\| \P_{ \{I_{(-\infty,t_j]}(Y_{t j}(\param))\}_{j=1}^{\vc}, \DataX_{t \vc}} - \P_{ \{I_{(-\infty,t_j]}(Y_{t j}(\param^{\prime}))\}_{j=1}^{\vc}, \DataX_{t \vc}} \|
\leq \| \P_{\Data^{t}_{\vc}(\param)} - \P_{\Data^{t}_{\vc}(\param^{\prime})} \|
\end{equation*}
completes the proof.
\qed
\end{proof}

We are now ready for the proof of Theorem~\ref{thm:consistency}.

\begin{proof}[Proof of Theorem~\ref{thm:consistency}]
The estimator $\hat{\param}_{T\TruParam}$ we will use is precisely the minimum-distance skeleton estimate of $\P_{\Data^{t}_{\vc}(\TruParam)}$ \cite{yatracos:85,devroye:01}.
\cite{yatracos:85} proved that if $N(\eps)$ is the $\eps$-covering number of $\{\P_{\Data^{t}_{\vc}(\TruParam)} : \param \in \Params\}$,
then taking this $\hat{\param}_{T\TruParam}$ estimator, then for some $T_{\eps} = O((1/\eps^{2})\log N(\eps/4))$, any $T \geq T_{\eps}$
has 
\[
\E\left[ \| \P_{\Data^{t}_{\vc}(\hat{\param}_{T\TruParam})} - \P_{\Data^{t}_{\vc}(\TruParam)} \| \right] < \eps.
\]
Thus, taking $G_{T} = \inf\{ \eps > 0 : T \geq T_{\eps} \}$, we have
\[
\E\left[ \| \P_{\Data^{t}_{\vc}(\hat{\param}_{T\TruParam})} - \P_{\Data^{t}_{\vc}(\TruParam)} \| \right] \leq G_{T} = o(1).
\]
Letting $R^{\prime}(T,\alpha)$ be any positive sequence with $G_{T} \ll R^{\prime}(T,\alpha) \ll 1$ and $R^{\prime}(T,\alpha) \geq G_{T} / \alpha$, 
and letting $\delta(T,\alpha) = G_{T} / R^{\prime}(T,\alpha) = o(1)$,
Markov's inequality implies
\begin{equation}
\label{eqn:d-dim-consistency}
\P\left( \| \P_{\Data^{t}_{\vc}(\hat{\param}_{T\TruParam})} - \P_{\Data^{t}_{\vc}(\TruParam)} \| > R^{\prime}(T,\alpha) \right) \leq \delta(T,\alpha) \leq \alpha.
\end{equation}
Letting $R(T,\alpha) = \min_{k} \left(M_{k}\left( R^{\prime}(T,\alpha) \right) + r_{k}\right)$,
since $R^{\prime}(T,\alpha) = o(1)$ and $r_{k} = o(1)$, we have $R(T,\alpha) = o(1)$.  Furthermore, 
composing \eqref{eqn:d-dim-consistency} with Lemmas~\ref{lem:prior-to-infty}, \ref{lem:infty-to-k}, and \ref{lem:k-to-d},
we have
\[
\P\left( \| \prior_{\hat{\param}_{T\TruParam}} - \prior_{\TruParam} \| > R(T,\alpha) \right) \leq \delta(T,\alpha) \leq \alpha.
\]
\qed
\end{proof}

\paragraph{Remark:} Although the above proof makes use of the minimum-distance skeleton estimator, 
which is typically not computationally efficient, it is often possible to achieve this same result (for certain families of distributions)
using a simpler estimator, such as the maximum likelihood estimator.  All we require is that the risk of the 
estimator converges to $0$ at a known rate that is independent of $\TruParam$.  For instance, see \cite{van-de-geer:00} for 
conditions on the family of distributions sufficient for this to be true of the maximum likelihood estimator.

\bibliographystyle{splncs03}
\bibliography{bib_atl_prior_convergence}

\end{document}